\documentclass{article}
\usepackage[lined,boxed,commentsnumbered]{algorithm2e}
\usepackage{subfigure,booktabs}
\usepackage{graphicx}\usepackage{subfigure}
\usepackage{cite}
\usepackage{amscd,amssymb} 
\usepackage{graphicx} \usepackage{array} \usepackage{subfigure}
\usepackage{psfrag}
\usepackage{xy} 
\usepackage{amsmath} 
\usepackage{fancyhdr}
\usepackage{setspace} 
\usepackage{rotating,epsfig} \usepackage{epsfig}
\usepackage{verbatim} \usepackage{multicol} \usepackage{longtable}
\usepackage{amsthm}
\newtheorem{thm}{Theorem}[section]
\newtheorem{cor}[thm]{Corollary}
\newtheorem{lem}[thm]{Lemma}

\newtheorem{defn}[thm]{Definition}
\theoremstyle{remark}

\newtheorem{ex}[thm]{\textbf{Example}}
\newcommand{\norm}[1]{\left\Vert#1\right\Vert}
\newcommand{\abs}[1]{\left\vert#1\right\vert}

\newcommand{\Real}{\mathbb R}

\newcommand{\nat}{\mathbb N}

\newcommand{\argmax}{\text{argmax}}

\newcommand{\sgn}{\ensuremath{\text{sgn}}}

\newcommand{\vc}[1]{#1}

\newcommand{\dif}[2]{{\operatorname{d}\over\operatorname{d}#2}#1}

\newcommand{\expect}[1]{\ensuremath{ \langle#1\rangle } }
\newcommand{\cov}[2]{\ensuremath{ \text{cov}(#1,#2) } }

\newcommand{\state}{\ensuremath{ \vc x}} 

\newcommand{\agi}{\ensuremath{ \mathfrak a}} 
\newcommand{\agii}{\ensuremath{ \mathfrak r}} 
\newcommand{\agiii}{\ensuremath{ \mathfrak q}} 

\newcommand{\agset}{\ensuremath{ \mathfrak A}} 

\newcommand{\collevent}{\ensuremath{\mathfrak C}}

\newcommand{\decke}{\ensuremath{\mathfrak u}}
\newcommand{\boden}{\ensuremath{\mathfrak l}}

\renewcommand{\Pr}{\mathrm{Pr}}

\newcommand{\beq}{\begin{equation}}
\newcommand{\eeq}{\end{equation}}

\title{Conservative collision prediction and avoidance for stochastic trajectories in continuous time and space} 

\author{Jan-P. Calliess, Michael Osborne, Stephen Roberts\footnote{The authours gratefully acknowledge funds via EPSRC EP/I011587.} \\
Dept. of Engineering Science, University of Oxford\\} 
					
\begin{document} 
 
\maketitle 
\begin{abstract}
\begin{quote}
Existing work in multi-agent collision prediction and avoidance typically assumes discrete-time trajectories with Gaussian uncertainty or that are completely deterministic. 
We propose an approach that allows
detection of collisions even between continuous, stochastic trajectories with the only restriction that means and covariances can be computed. 
To this end, we employ probabilistic bounds to derive criterion functions whose negative sign provably is
indicative of probable collisions. For criterion functions that are Lipschitz, an algorithm is provided to rapidly find negative values or prove their absence.
We propose an iterative policy-search approach that avoids prior discretisations and, upon termination, yields collision-free
trajectories with adjustably high certainty. We test our method
with both fixed-priority and auction-based protocols for coordinating
the iterative planning process. Results are provided in
collision-avoidance simulations of feedback controlled plants.
\end{quote}
\end{abstract}

\noindent 

\section{Introduction}

Due to their practical importance, multi-agent collision avoidance and control have been extensively studied across different communities including AI, robotics and control. Considering continuous stochastic trajectories, reflecting each agent's uncertainty about its neighbours' time-indexed locations in an environment space, we exploit a distribution-independent bound on collision probabilities to develop a conservative collision-prediction module. It avoids temporal discretisation by stating collision-prediction as a one-dimensional optimization problem. If mean and standard deviation are computable Lipschitz functions of time, one can derive Lipschitz constants that allow us to guarantee collision prediction success with low computational effort. This is often the case, for instance, when dynamic knowledge of the involved trajectories is available 
(e.g. maximum velocities or even the stochastic differential equations).
 
To avoid collisions detected by the prediction module, we let an agent re-plan repeatedly until no more collisions occur with a definable probability. Here, re-planning refers to modifying a control signal (influencing the basin of attraction and equilibrium point of the agent's stochastic dynamics) so as to bound the collision probability while seeking low plan execution cost in expectation. To keep the exposition concrete, we focus our descriptions on an example scenario where the plans correspond to sequences of setpoints of a feedback controller regulating an agent's noisy state trajectory. However, one can apply our method in the context of more general policy search problems.

In order to foster low social cost across the entire agent collective, we compare two different coordination mechanisms. Firstly, we consider a simple fixed-priority scheme \cite{erdmann_movingobj:87}, and secondly, we modify an auction-based coordination protocol \cite{ArmsTR:2011} to work in our continuous setting. In contrast to pre-existing work in auction-style multi-agent planning (e.g. \cite{ArmsTR:2011,koenig_auction_guarantees}) and multi-agent collision avoidance (e.g. \cite{kostic2010collision,ayanian2010decentralized, vandenberg:12}), we avoid {\it a priori} discretizations of space and time. Instead, we recast the coordination problem as one of incremental open-loop policy search. That is, as a succession of continuous optimisation or root-finding problems that can be efficiently and reliably solved by modern optimisation and root-finding techniques (e.g. \cite{Shubert:72,direct:93}).

While our current experiments were conducted with linear stochastic differential equation (SDE) models with state-independent noise (yielding Gaussian processes), our method is also applicable to any situation where mean and covariances can be evaluated. This encompasses non-linear, non-Gaussian cases that may have state-dependent uncertainties (cf. \cite{Gardiner2009}). 

This preprint is an extended and improved version of a conference paper that appeared in \textit{Proc. of the 13th International Conference on Autonomous Agents and Multiagent Systems (AAMAS 2014)} \cite{CalliessAAMAS2014}.

\subsection{Related Work}

Multi-agent trajectory planning and task allocation methods have been related to auction mechanisms by identifying locations in state space with atomic goods to be auctioned in a sequence of repeated coordination rounds (e.g. \cite{ArmsTR:2011,koenig_auction_guarantees,koenig_biding_rules_gen}). Unfortunately, even in finite domains the coordination is known to be intractable -- for instance the sequential allocation problem is known to be NP-hard in the number of goods and agents \cite{sandholmcombauc:2002,Koenig2007}. Furthermore, collision avoidance corresponds to non-convex interactions.

This renders the coordination problem inapplicable to standard optimization techniques 
that rely on convexity of the joint state space. In recent years, several works have investigated the use of mixed-integer programming techniques 
for single- and multi-agent model-predictive control with collision avoidance both in deterministic and stochastic settings \cite{ArmsTR:2011,LyonsACC2012}. 
To connect the problem to pre-existing mixed-integer optimization tools these works had to limit the models to dynamics governed by linear, time-discrete difference equations with state-independent state noise. The resulting plans were finite sequences of control inputs that could be chosen freely from a convex set. The controls gained from optimization are open-loop -- to obtain closed-loop policies the optimization problems have to be successively re-solved on-line in a receding horizon fashion. However, computational effort may prohibit such an approach in multi-agent systems with rapidly evolving states. 

Furthermore, prior time-discretisation comes with a natural trade-off. On the one hand, one would desire a high temporal resolution in order to limit the chance of missing a collision predictably occurring between consecutive time steps. On the other hand, communication restrictions, as well as poor scalability of mixed-integer programming techniques in the dimensionality of the input vectors, impose severe restrictions on this resolution. To address this trade-off,
 \cite{Earl2005} proposed to interpolate between the optimized time steps in order to detect collisions occurring between the discrete time-steps. Whenever a collision was detected they proposed to augment the temporal resolution by the time-step of the detected collision thereby growing the state-vectors incrementally as needed. A detected conflict, at time $t$, is then resolved by solving a new mixed-integer linear programme over an augmented state space, now including the state at $t$. This approach can result in a succession of solution attempts of optimization problems of increasing complexity, but can nonetheless prove relatively computationally efficient. Unfortunately, their method is limited to linear, deterministic state-dynamics.

Another thread of works relies on dividing space into polytopes \cite{li2007motion,ayanian2010decentralized}, while still
others \cite{chang2003collision,dimarogonas2006feedback,mastellone2008formation,kostic2010collision} adopt a potential field. In not accommodating uncertainty and stochasticity, these approaches are forced to be overly conservative in order to prevent collisions in real systems.

In contrast to all these works, we will consider a different scenario. Our exposition focuses on the assumption that each agent is regulated by influencing its continuous stochastic dynamics. For instance, we might have a given feedback controller with which one can interact by providing a sequence of setpoints constituting the agent's plan. While this restricts the choice of control action, it also simplifies computation as the feedback law is fixed. The controller can generate a continuous, state-dependent control signal based on a discrete number of control decisions, embodied by the setpoints. Moreover, it renders our method applicable in settings where the agents' plants are controlled by standard off-the-shelf controllers (such as the omnipresent PID-controllers) rather than by more sophisticated customized ones.
Instead of imposing discreteness, we make the often more realistic assumption that agents follow continuous time-state trajectories within a given continuous time interval. Unlike most work \cite{stipanovic2007cooperative,van2008reciprocal,mastellone2008formation,ayanian2010decentralized} in this field, we allow for stochastic dynamics, where each agent cannot be certain about the location of its team-members. This is crucial for many real-world multi-agent systems. The uncertainties are modelled as state-noise which can reflect physical disturbances or merely model inaccuracies.
While our exposition's focus is on stochastic differential equations, our approach is generally applicable in all contexts where the first two moments of the predicted trajectories can be evaluated for all time-steps.
As noted above, this paper is an extended version of work that has been published in the proceedings of AAMAS'14 \cite{CalliessAAMAS2014} and an earlier stage of this work was presented at an ICML \cite{CalliessICML2012} workshop.
\section{Predictive Probabilistic Collision Detection with Criterion Functions} \label{sec:colldetection}

 \textbf{Task}. Our aim is to design a collision-detection module that can decide whether a set of 
(predictive) stochastic trajectories is collision-free (in the sense defined below). The module we will derive is guaranteed to make this decision correctly, based on knowledge of the first and second order moments of the trajectories alone. In particular, no assumptions are made about the family of stochastic processes the trajectories belong to. As the required collision probabilities will generally 
have to be expressed as non-analytic integrals, we will content ourselves with a fast, \textit{conservative} approach. That is, we are willing to tolerate a non-zero false-alarm-rate as long as decisions can be made rapidly and with zero false-negative rate. Of course, for certain distributions and plant shapes, one may derive closed-form solutions for the collision probability that may be less conservative and hence, lead to faster termination and shorter paths. In such cases, our derivations can serve as a template for the construction of criterion functions on the basis of the tighter probabilistic bounds. 

\textbf{Problem Formalization}. Formally, a collision between two  objects (or agents) $\agi,\agii$ at time $t \in I := [t_0,t_f] \subset \Real$ can be described by the event 

$\mathfrak C^{\agi,\agii}(t) $ $ = \{ (\state^\agi(t),\state^\agii(t)) | \norm{\state^\agi(t)-\state^\agii(t)}_2 \leq \frac{\Lambda^\agi + \Lambda^\agii}{2}  \}$. Here, $\Lambda^\agi,\Lambda^\agii$ denote the objects' diameters, and $x^\agi,x^\agii : I  \to \Real^D$ are two (possibly uncertain) trajectories in a common, $D$-dimensional interaction space.

In a stochastic setting, we desire to bound the collision probability below a threshold $\delta \in (0,1)$ at any given time in I. We loosely say that the trajectories are \textit{collision-free }if $\Pr[\mathfrak C^{\agi,\agii}(t)]  < \delta, \forall t \in I$.

\textbf{Approach.} For conservative collision detection between two agents' stochastic trajectories $\state^\agi,\state^\agii$, we 
construct a \textit{criterion function} $\gamma^{\agi,\agii} : I \to \Real$ (eq. as per Eq. \ref{eq:critfctgeneric} below). A conservative criterion function has the property $\gamma^{\agi,\agii}(t) >0 \Rightarrow \Pr [\mathfrak C^{\agi,\agii}(t)]  < \delta^\agi$. That is, a collision between the
trajectories with probability above $\delta$ can be ruled-out if $\gamma^{\agi,\agii}$ attains only positive values. 
If one could evaluate the function $t \mapsto \Pr [\mathfrak C^{\agi,\agii}(t)]$, an ideal criterion function would be 
\begin{equation}\label{eq:collcritideal}
	\gamma^{\agi,\agii}_{\text{ideal}}(t) := \delta - \Pr [\mathfrak C^{\agi,\agii}(t)].
\end{equation}
It is ideal in the sense that $\gamma^{\agi,\agii}_{\text{ideal}}(t) >0 \Leftrightarrow \Pr [\mathfrak C^{\agi,\agii}(t)] < \delta$.
However, in most cases, evaluating the criterion function in closed form will not be feasible. Therefore, we adopt a conservative approach: That is, we determine a criterion function $
	\gamma^{\agi,\agii}(t) $ such that provably, we have $\gamma^{\agi,\agii}(t) \leq \gamma^{\agi,\agii}_{\text{ideal}}(t), \forall t $, including the possibility of false-alarms. That is, it is possible that for some times $t$, $\gamma^{\agi,\agii}(t)  \leq 0$, in spite of $\gamma^{\agi,\agii}_{\text{ideal}}(t) > 0$.

Utilising the conservative criterion functions for collision-prediction, we assume a collision occurs unless $\min_{t \in I} \gamma^{\agi,\agii}(t) >0,\forall \agii \neq \agi$. If the trajectories' means and standard deviations are Lipschitz functions of time then one can often show that $\gamma^{\agi,\agii}$ is Lipschitz as well. In such cases negative values of $\gamma^{\agi,\agii}$ can be found or ruled out rapidly, as will be discussed in Sec. 
\ref{Sec:lipfctnegvalsfind}.
In situations where a Lipschitz constant is unavailable or hard to determine, we can base our detection on the output of a global minimization method such as DIRECT \cite{direct:93}.

\subsection{Finding negative function values of Lipschitz functions}
\label{Sec:lipfctnegvalsfind}

Let $t_0,t_f \in \Real, t_0 \leq t_f, I := [t_0,t_f] \subset \Real$. Assume we are given a
Lipschitz continuous \emph{target function} $f: I \to \Real $ with Lipschitz constant
$L \geq 0$. That is, $\forall S
\subset I \,\exists L_S \leq L \,\forall x,x' \in S: \abs{f(x) -
f(x')} \leq L_S \, \abs{x-x'}$. Let $t_0 < t_1 < t_2 <...<t_N < t_f $
and define $G_N = ( t_0,\ldots,t_{N+1})$ to be the \emph{sample grid} of
size $N+2 \geq 2$ consisting of the inputs at which we choose to evaluate the
target $f$.

\emph{Our goal is to prove or disprove the existence of a negative function value of target $f$}.

\subsubsection{A naive algorithm}
As a first, naive method, Alg. \ref{alg:negdetect_lipschitz} leverages Lipschitz continuity to answer the question of positivity correctly after a finite number of function 
evaluations.

\begin{algorithm}
\SetKwData{flag}{flag} \SetKwData{negTime}{criticalTime} \SetKwData{minVal}{minVal}
\SetKwData{grid}{TimeGrid} \SetKwData{lipschitzconst}{L}
\SetKwFunction{OR}{OR}\SetKwFunction{FindCompress}{FindCompress} \SetKwFunction{Resolve}{Resolve}
\SetKwFunction{insert}{Insert} \SetKwFunction{Planner}{Planner} \SetKwFunction{Auction}{Auction}
\SetKwFunction{Receive}{Receive} \SetKwFunction{Avoid}{Avoid}
\SetKwFunction{DetectCollisions}{CollDetect}
\SetKwInOut{Input}{input}\SetKwInOut{Output}{output}
\Input{Domain boundaries $t_0,t_f$ $\in \Real$, function  $\gamma: (t_0,t_f) \to \Real$, Lipschitz constant $\lipschitzconst >0$.}
\Output{Flag \flag  indicating presence of a non-positive function value (\flag = 1 indicates existence of a non-positive function value; \flag =0 indicates it has been ruled out that a negative function value can exist). Variable \negTime contains the time of a non-positive function value if such exists (\negTime $=t_0-1$, 
iff $\gamma((t_0,t_f)) \subset \Real_{+}$).}
\BlankLine

$\flag \leftarrow -1$;
$\negTime \leftarrow t_0-1$;

$\grid \leftarrow \{t_0,t_f\}$;

$r \leftarrow -1$;

\Repeat{$ \flag =1 $ \OR $\flag =0$ }{
$r \leftarrow r +1$;
$\Delta \leftarrow \frac{t_f-t_0}{2^r}$;
$ N \leftarrow (t_f-t_0)/\Delta$;

$\grid \leftarrow \cup_{i=0}^{N} \{t_0+i \Delta\}$;

$\minVal \leftarrow \min_{t \in \grid} \gamma(t);$ \\
\uIf{$\minVal \leq 0$}{
$\flag \leftarrow 1$;
$\negTime \leftarrow \arg\min_{t \in \grid} \gamma(t)$;
}
\uElseIf{\minVal $> \lipschitzconst \, \Delta$ }
{$\flag \leftarrow 0$;}
}

\caption{Naive algorithm deciding whether a Lipschitz continuous function $\gamma$ has a non-positive value on a compact domain. 
Note, if minVal $>$ L $\, \Delta$ the function is guaranteed to map into the positive reals exclusively.}
\label{alg:negdetect_lipschitz}
\end{algorithm}

The algorithm evaluates the function values on a finite grid assuming a uniform constant Lipschitz number $L$. The grid is iteratively refined until either a negative function value is found or, the
Lipschitz continuity of function $\gamma$ allows us to infer that no negative function values can exist. The latter is the case whenever $\min_{t \in G_N} \gamma(t) > L \, \Delta$ where 
$G_N = (t_0,..., t_{N+1})$ is the grid of function input (time) samples, $\Delta = |t_{i+1} - t_i| \, (i=0,...,N-1)$ and $L >0$ a Lipschitz number
of the function $\gamma: (t_0,t_f) \to \Real$ which is to be evaluated. 

The claim is established by the following Lemma:

\begin{lem}
 Let $\gamma: [t_0,t_f]\subset \Real \to \Real$ be a Lipschitz function with Lipschitz number $L>0$. Furthermore, let $G_N = (t_0,t_1,\ldots, t_{N+1})$ be 
an equidistant grid with $\Delta  = |t_{i+1} - t_i| \, (i=0,...,N-1)$.  

We have, $\gamma(t) > 0, \forall t \in (t_0,t_f)$ 
if 
$\forall t \in G_N: \gamma(t) > L \, \Delta $.

\begin{proof}
 Since $L$ is a Lipschitz constant of $\gamma$ we have $|\gamma(t) - \gamma(t') | \leq L |t-t'|, \forall t,t' \in (t_0,t_f) $. Now, let $t^* \in (t_0,t_f)$  and 
  $t_i, t_{i+1} \in G_N$ such that $t^* \in [t_i,t_{i+1}]$. Consistent with the premise of the implication we aim to show, we assume 
 $\gamma(t_i), \gamma(t_{i+1}) > L \Delta $ and, without loss of generality, we assume $\gamma(t_i)\leq \gamma(t_{i+1})$. Let $\delta  := |t_i -t^*|$. 
 Since $t_i \leq t^* \leq t_{i+1}$ we have $0 \leq \Delta - \delta $. Finally, $0 < L \Delta  < |\gamma(t_i)|$ implies $ \gamma(t^*) \geq  \gamma(t_i) - 
|\gamma(t_i) - \gamma(t^*)| \geq \gamma(t_i) - 
L |t_i - t^*| > L \Delta   - L \delta  = L (\Delta   -  \delta ) \geq 0$.
\end{proof}

\end{lem}

Appart from a termination criterion, the lemma establishes that larger Lipschitz numbers will generally cause longer run-times of the algorithm as finer resolutions $\Delta t$ will be required to ensure non-negativity of the function under investigation.

\subsubsection{An improved adaptive algorithm} \label{sec:adaptiveLipshubertstyle}
Next, we will present an improved version of the algorithm provided above.
 We can define two functions, \emph{ceiling} $\decke_N$ and \emph{floor} $ \boden_N$, such that (i) they bound the target $\forall t \in I: \boden_N(t) \leq \gamma(t) \leq \decke_N(t)$, and (ii) the bounds get tighter for denser grids. In particular, one can show that $\boden_N , \decke_N \stackrel{N \to \infty} {\longrightarrow} f $ uniformly if $G_N$ converges to a dense subset of $[a,b]$. 
Define $\xi^{\mathfrak l}_N := \arg \min_{x \in I} \boden_{N}(x)$. 
It has been shown that $\xi^{\mathfrak l}_N  =\min_{i=1}^{N-1} \frac{t_{i+1}+t_i}{2} - \frac{\gamma(t_{i+1}) - \gamma(t_i)}{2 L}$ and  
$\boden_N(\xi^{\mathfrak l}_N) = \min_i \frac{\gamma(t_{i+1}) +\gamma(t_i)}{2} - L \frac{t_{i+1}-t_i}{2}$ (see \cite{Shubert:72,direct:93}).
It is trivial to refine this to take localised Lipschitz constants into account: 
$\xi^{\mathfrak l}_N = \min_i \frac{\gamma(t_{i+1}) +\gamma(t_i)}{2} - L_{J_i} \frac{t_{i+1}-t_i}{2}$ where $L_{J_i}$ is a Lipschitz number valid on interval $J_i = (t_i,t_{i+1})$.

This suggests the following algorithm: \textit{We refine the grid $G_N$ to grid $G_{N+1}$, by including $\xi^\boden_N, f(\xi^\boden_N)$ as a new sample. This process is repeated until either of the following stopping conditions are met: (i) a negative function value of $\gamma$ is discovered ($f(\xi^{\mathfrak l}_N) < 0$), or (ii) $\boden_N(\xi^{\mathfrak l}_N) \geq 0$ (in which case we are guaranteed that no negative function values can exist)}. 
\begin{figure*}
        \centering
        \begin{subfigure}
                \centering
                \includegraphics[width = 3.9cm, clip, trim = 3.5cm 9.5cm 4.5cm 10cm]{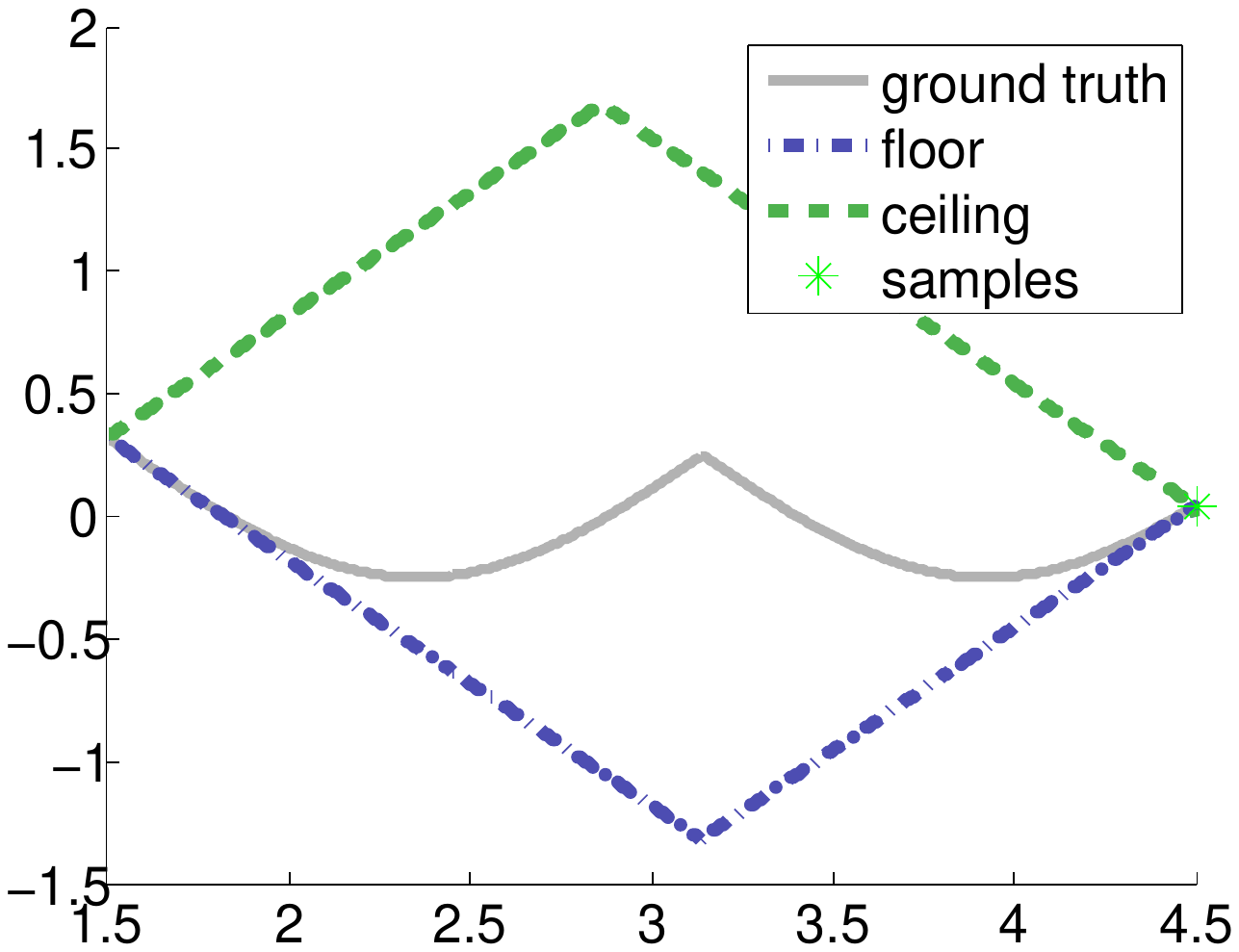}
        \end{subfigure}%
        \begin{subfigure}
                \centering
                \includegraphics[width = 3.9cm, clip, trim = 3.5cm 9.5cm 4.5cm 10cm]{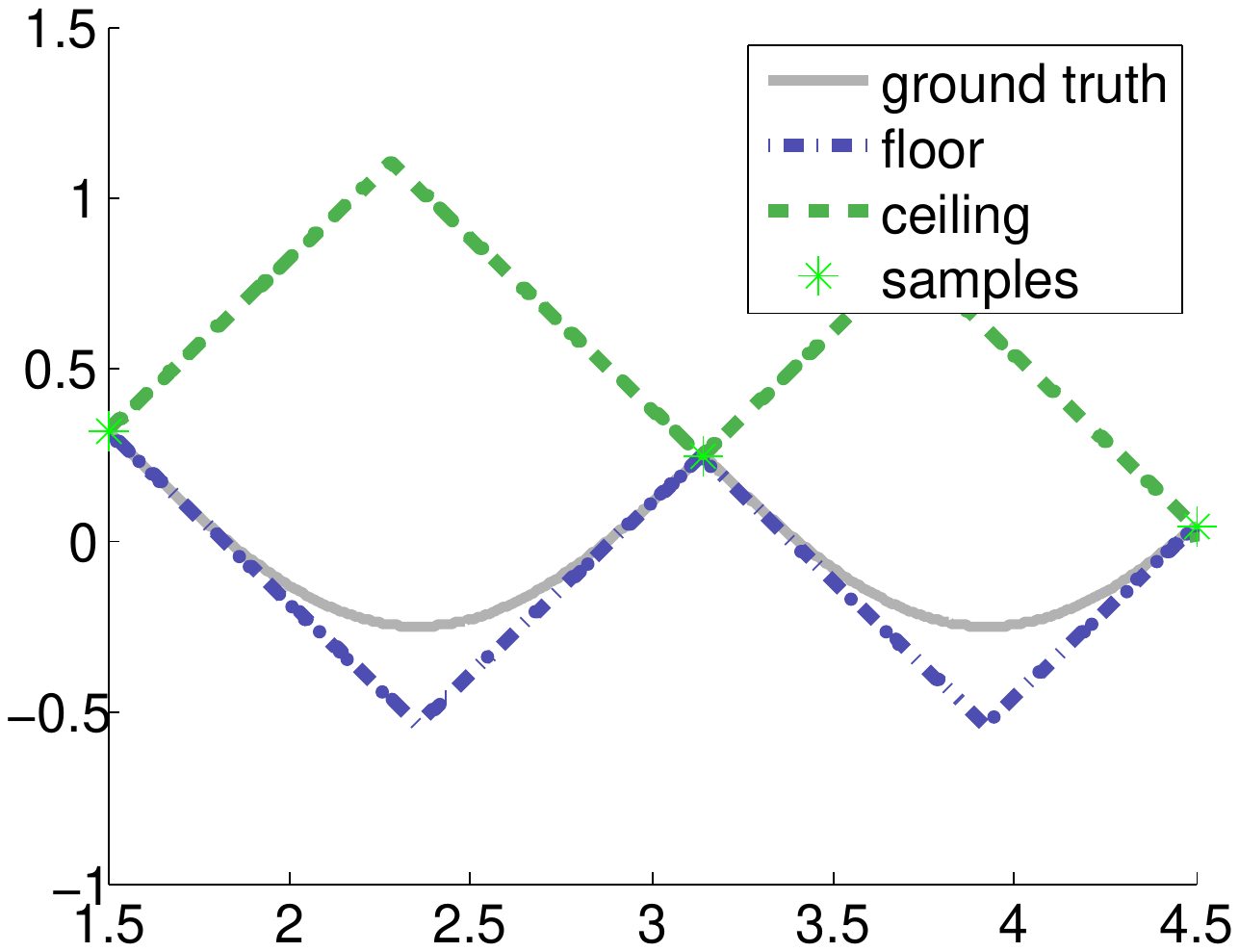}
        \end{subfigure}
        \begin{subfigure}
                \centering
                \includegraphics[width = 3.9cm, clip, trim = 3.5cm 9.5cm 4.5cm 10cm]{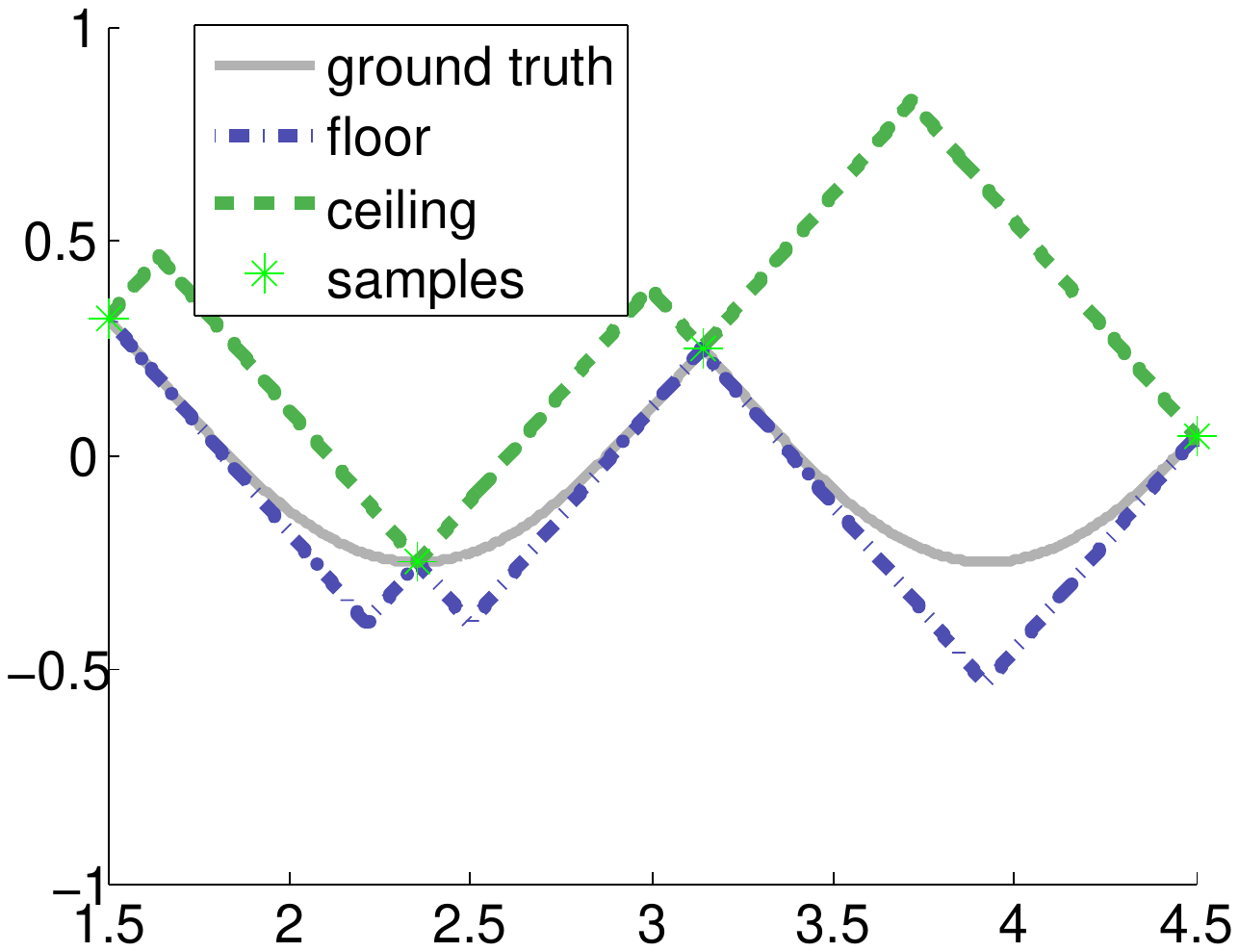}
        \end{subfigure}
       \caption{Proving the existence of a negative value of function 
			$x \mapsto \abs{\sin(x)}  \cos(x) +\frac{1}{4}$. Left: Initial condition. Centre: First refinement. Right: The second refinement has revealed the existence of a negative value.}
       \label{fig:shubert}
\end{figure*}

\begin{algorithm}
\SetKwData{flag}{flag} \SetKwData{negTime}{criticalTime} \SetKwData{minVal}{minVal}
\SetKwData{grid}{$G_N$} \SetKwData{lipschitzconst}{L}
\SetKwFunction{OR}{OR}\SetKwFunction{FindCompress}{FindCompress} \SetKwFunction{Resolve}{Resolve}
\SetKwFunction{insert}{Insert} \SetKwFunction{Planner}{Planner} \SetKwFunction{Auction}{Auction}
\SetKwFunction{Receive}{Receive} \SetKwFunction{Avoid}{Avoid}
\SetKwFunction{DetectCollisions}{CollDetect}
\SetKwInOut{Input}{input}\SetKwInOut{Output}{output}
\Input{Domain boundaries $t_0,t_f$ $\in \Real$, function  $\gamma: (t_0,t_f) \to \Real$, Lipschitz constant $\lipschitzconst >0$.}
\Output{Flag \flag  indicating presence of a non-positive function value (\flag = 1 indicates existence of a non-positive function value; \flag =0 indicates it has been ruled out that a negative function value can exist). Variable \negTime contains the time of a non-positive function value if such exists (\negTime $=t_0-1$, 
iff $\gamma((t_0,t_f)) \subset \Real_{+}$).}
\BlankLine

$\flag \leftarrow -1$;
$\negTime \leftarrow t_0-1$;

$\grid \leftarrow \{t_0,t_f\}$;

$N=0$;

\Repeat{$ \flag =1 $ \OR $\flag = 0$ }{

$\xi^{\mathfrak l}  \leftarrow \min_{i=1}^{N} \frac{t_{i+1}+t_i}{2} - \frac{\gamma(t_{i+1}) - \gamma(t_i)}{2 L}$; \\

$\boden_N(\xi^{\mathfrak l}_N) \leftarrow \min_{i=1}^{N} \frac{\gamma(t_{i+1}) +\gamma(t_i)}{2} - L \frac{t_{i+1}-t_i}{2}$;

$\minVal \leftarrow \gamma(\xi^{\mathfrak l});$ \\

\uIf{$\minVal \leq 0$}{
$\flag \leftarrow 1$;
$\negTime \leftarrow \xi^\boden$;
}
\uElseIf{$\boden_N(\xi^{\mathfrak l}_N) > 0$}{
$\flag \leftarrow 0$;
}
\Else{
$N \leftarrow N +1$; 
$\grid \leftarrow \grid \cup \{\xi^{\mathfrak l} \}$;

}
}

\caption{Adaptive algorithm based on Shubert's method to prove whether a Lipschitz continuous function $\gamma$ has a non-positive value on a compact domain. 
Note, if $\boden_N(\xi^{\mathfrak l}_N) >0$ the function is guaranteed to map into the positive reals exclusively.}
\label{alg:negdetect_lipschitz_shubertstyle}
\end{algorithm}

For pseudo-code refer to Alg. \ref{alg:negdetect_lipschitz_shubertstyle}.

An example run is depicted in Fig. \ref{fig:shubert}. Note, without our stopping criteria, our algorithm degenerates to Shubert's minimization method \cite{Shubert:72}. The stopping criteria are important to save computation, especially in the absence of negative function values.

\subsection{Deriving collision criterion functions}

This subsection is dedicated to the derivation of a (Lipschitz) criterion function. In lieu to the approach of \cite{ArmsTR:2011,Lyons2011}, the idea is to define hyper-cuboids $H^\agi, H^\agii$ sufficently large to contain a large enough proportion of each agent's probability mass to ensure that no collision occurs (with sufficient confidence) as long as the cuboids do not overlap. We then define the criterion function so as to negative values whenever the hyper-cuboids do overlap.

For ease of notation, we omit the time index $t$. For instance, in this subsection, $x^\agi$ now denotes random variable $x^\agi(t)$ rather than the stochastic trajectory.

The next thing we will do is to derive sufficient conditions for absence of collisions, i.e. for $\Pr[\collevent^{\agi,\agii}  ] < \delta$.

To this end, we make an intermediate step: 
For each agent $\agiii \in \{\agi,\agii\}$ we define an open hyper-cuboid $H^\agiii$ centred around mean $\mu^\agiii = \expect{\state^\agiii(t)}$. As a $D$-dimensional 
hyper-cuboid, $H^\agiii$ is completely determined by its centre point $\mu^\agiii$ and its edge lengths $l^\agiii_1,...,l^\agiii_D$.  
Let $O^\agiii$ denote the event that $x^\agiii \notin H^\agiii$ and $P^\agiii := \Pr[O^\agiii]$. We derive a simple disjunctive constraint 
on the component distances of the means under which we can guarantee that the collision probability is not greater than the probability 
of at least one object being outside its hyper-cuboid. This is the case if the hypercuboids do not overlap. That is, their max-norm distance is at least 
$\Lambda^{\agi,\agii} : = \frac{\Lambda^\agi + \Lambda^\agii}{2}$.

Before engaging in a formal discussion we need to establish a preparatory fact:

\begin{lem}\label{lem:star}
Let $\mu^\agiii_j$ denote the $j$th component of object $\agiii$'s mean and $r_j^\agiii = \frac 1 2 l_j^\agiii$. 
Furthermore, let $\mathfrak F^{\agi,\agii} := \overline {\collevent^{\agi,\agii}} $ be the event that no collision occurs and $\mathfrak B^{\agi,\agii}:= H^\agi \times H^\agii$ the event that 
$x^\agi \in H^\agi$ and $x^\agii \in H^\agii$.
Assume the component-wise distance between the hyper-cuboids $H^\agi,H^\agii$ is at least $\Lambda^{\agi,\agii}$, which is expressed by the following disjunctive constraint:
\[\exists j \in \{1,...,D\}: \abs{\mu^\agi_j - \mu_j^\agii } > \Lambda^{\agi,\agii} + r^\agi_j + r^\agii_j.\]

Then, we have : $ \mathfrak B^{\agi,\agii} \subset \mathfrak F^{\agi,\agii}.$

\begin{proof}
 Since $\norm{x}_\infty \leq \norm{x}_2,  \forall x$ we have 

$\mathfrak F_\infty := \{(x^\agi,x^\agii) \vert \norm{x^\agi - x^\agii}_\infty> \Lambda^{\agi,\agii} \}$ $\subset \{(x^\agi,x^\agii) \vert \norm{x^\agi - x^\agii}_2 > \Lambda^{\agi,\agii} \} = \mathfrak F^{\agi,\agii} $.
It remains to be shown that $\mathfrak B^{\agi,\agii} \subset \mathfrak F_\infty$:
Let $(x^\agi, x^\agii) \in \mathfrak B^{\agi,\agii} = H^\agi \times H^\agi$. Thus, 
$\forall j \in \{1,...,D\}, \agiii \in \{\agi,\agii\}: \abs{x^\agiii_j - \mu^\agiii_j} \leq r^\agiii_j$. 
For contradiction, assume $(x^\agi, x^\agii) \notin \mathfrak F_\infty$. Then, $\abs{x^\agi_i -x^\agii_i} \leq \Lambda^{\agi,\agii}$
for all $i \in \{1,...,D\}$. 

Hence, $\abs{\mu^\agi_i - \mu^\agii_i} = \abs{\mu^\agi_i - x^\agi_i + x^\agi_i - x^\agii_i + x^\agii_i- \mu^\agii_i}$
$\leq \abs{\mu^\agi_i - x^\agi_i} + \abs{x^\agi_i - x^\agii_i} + \abs{x^\agii_i- \mu^\agii_i} \leq r^\agi_i + \Lambda^{\agi,\agii} + r^\agii_i, \forall i \in \{1,...,D\}$ 
which contradicts our disjunctive constraint in the premise of the lemma. q.e.d.

\end{proof}

\end{lem}

\begin{thm}
\label{thm:hypercubprobsconstr}
Let $\mu^\agiii_j$ denote the $j$th component of object $\agiii$'s mean and $r_j^\agiii = \frac 1 2 l_j^\agiii$. Assume, $\state^\agi, \state^\agii$ are random variables with means $\mu^\agi = \expect{\state^\agi},\mu^\agii = \expect{\state^\agii}$, respectively. The max-norm distance between hypercuboids $H^\agi,H^\agii$ is at least $\Lambda^{\agi,\agii} > 0$ (i.e. the hypercuboids do not overlap), which is expressed by the following disjunctive constraint:
\[\exists j \in \{1,...,D\}: \abs{\mu^\agi_j - \mu_j^\agii } > \Lambda^{\agi,\agii} + r^\agi_j + r^\agii_j.\]

Then, we have : \[\Pr[\collevent^{\agi,\agii}] \leq P^\agi + P^\agii - P^\agii \, P^\agi  \leq P^\agi + P^\agii\]
where $P^\agiii = \Pr[x^\agiii \notin H^\agiii], (\agiii \in \{\agi,\agii \})$.
%
\begin{proof} As in Lem. \ref{lem:star}, let $\mathfrak F^{\agi,\agii} := \overline {\collevent^{\agi,\agii}} $ be the event that no collision occurs and 
let $\mathfrak B^{\agi,\agii}:= H^\agi \times H^\agii$.
We have 
$\Pr[\collevent^{\agi,\agii}]$
$\leq 1 - \Pr[\overline {\collevent^{\agi,\agii} }] = 1- \Pr[\mathfrak F^{\agi,\agii}]$. 
 By Lem. \ref{lem:star} we have $\mathfrak B^{\agi,\agii} \subset \mathfrak F^{\agi,\agii}$ and thus, 
$ 1- \Pr[\mathfrak F^{\agi,\agii}] \leq 1- \Pr[\mathfrak B^{\agi,\agii}] = \Pr[\overline {\mathfrak B^{\agi,\agii}}]$. Now, $\Pr[\overline {\mathfrak B^{\agi,\agii}}] = \Pr[x^\agi \notin H^\agi \vee x^\agii \notin H^\agii  ]$
$= P^\agi + P^\agii - P^\agi \, P^\agii \leq P^\agi + P^\agii$. q.e.d.
\end{proof}

\end{thm}

One way to define a criterion function is as follows: 

\begin{equation}
\label{eq:critfctgeneric}
\gamma^{\agi,\agii} (t; \varrho(t)) := \max_{i=1,\ldots,D} \{\abs{\mu^\agi_i(t) - \mu_i^\agii(t) } -\Lambda^{\agi,\agii} - r^\agi_i(t) - r^\agii_i(t)\}
\end{equation}
where $\varrho = (r_1^\agi,\ldots,r_D^\agi,r_1^\agii,\ldots,r_D^\agii)$ is the parameter vector of radii. (For notational convenience, we will often omit explicit mention of parameter $\varrho$ in the function argument.)

For more than two agents, agent $\agi's$ overall criterion function is 
$\Gamma^\agi(t) := \min_{\agii \in \agset\backslash\{\agi\}} \gamma^{\agi,\agii} (t).$

Thm. \ref{thm:hypercubprobsconstr} tells us that the collision probability is bounded from above by the desired threshold $\delta$ if $\gamma^{\agi,\agii} (t) >0$, provided we chose the radii $r_j^\agi,r^\agii_j$ ($j=1,...,D$) such that 
$P^\agi, P^\agii \leq \frac{\delta}{2}$.

Let $\agiii \in \{\agi,\agii\}$.
Probability theory provides several distribution-independent bounds relating the radii of a (possibly partly unbounded) hypercuboid to 
the probability of not falling into it. That is, these are bounds of the form 

\[P^\agiii \leq \beta(r^\agiii_1,...,r^\agiii_D; \Theta)\] where $\beta$ is a continuous function that decreases monotonically with increasing radii and $\Theta$ represents additional information. 
In the case of Chebyshev-type bounds information about the first two moments are folded in, i.e. $\Theta = (\mu^\agiii, C^\agiii)$ where $C^\agiii (t) \in \Real^{D \times D}$  is the variance (-covariance) matrix.
We then solve for radii that fulfil the inequality $\frac{\delta}{2} \stackrel{}{ \geq} \beta(r^\agiii_1,...,r^\agiii_D; \Theta)$ while simultaneously ensuring collision avoidance with the desired probability.

Inspecting Eq. \ref{eq:critfctgeneric}, it becomes clear that, in order to maximally diminish conservatism of the criterion function, it would be ideal to choose the radii in $\varrho$ such that 
$\varrho = \argmax_\varrho \gamma^{\agi,\agii}(t;\varrho) = \argmax_{r_1^\agi,...,r^\agi_D\\r_1^\agii,...,r^\agii_D} \max_{i=1,\ldots,D} \{\abs{\mu^\agi_i - \mu_i^\agii } -\Lambda^{\agi,\agii} - r^\agi_i - r^\agii_i\}$ subject to the constraints $\frac{\delta}{2} \stackrel{}{ \geq} \beta(r^\agiii_1,...,r^\agiii_D; \Theta), (\agiii \in \{\agi,\agii\})$.
Solving this constrained optimisation problem can often be done in closed form.

In the context where $\beta$ is derived from a Chebyshev-type bound, we propose to set as many radii as large as possible (in order to decrease ($\beta$ to satisfy the constraints) while setting the radii $r_i^\agi, r_i^\agii$ as small as possible without violating the constraint (where $i$ is some dimension). 
That is, we define the radii as follows: Set $r_j^\agiii := \infty, \forall j \neq i$. The remaining unknown variable, $r_i^\agiii$, then is defined as the solution to the equation $\frac{\delta}{2} = \beta(r^\agiii_1,...,r^\agiii_D; \Theta)$.
The resulting criterion function, denoted by $\gamma^{\agi,\agii}_i$, we obtain with this procedure of course depends on the arbitrary choice of dimension $i$.
Therefore, we obtain a less conservative criterion function by repeating this process for each dimension $i$ and then constructing a new criterion function as the point-wise maximum: $\gamma^{\agi,\agii} (t):= \max_i \gamma^{\agi,\agii}_i(t)$.

A concrete example of this procedure is provided below.

\subsubsection{Example constructions of distribution-independent criterion functions}
We can use the above derivation as a template for generating criterion functions.

Consider the following concrete example. Combining union bound and the standard (one-dim.) Chebyshev bound yields 
$P^\agiii =\Pr [\state^\agiii \notin H^\agiii ] \leq \sum_{j=1}^D \frac{C_{jj}^\agiii}{ r_j^\agiii r_j^\agiii } =: \beta(r_1^\agiii,\ldots,r_D^\agiii ; C^\agiii)$.
Setting every radius, except $r_i^\agiii$, to infinitely large values and $\beta$ equal to $\frac{\delta}{2}$ yields 
$\frac{\delta}{2} = \frac{C_{ii}^\agiii}{ r_i^\agiii r_i^\agiii}$, i.e. $r_i^\agiii = \sqrt{\frac{2 C_{ii}^\agiii} {\delta}}$.
(Note, this a correction of the radius provided in the conference version of this paper.) 
Finally, inserting these radii ( for $\agiii = \agi,\agii$) into Eq. \ref {eq:critfctgeneric} yields our first collision criterion function:
%
$\gamma^{\agi,\agii} (t) := \abs{\mu^\agi_i(t) - \mu_i^\agii(t) } -\Lambda^{\agi,\agii} - \sqrt{\frac{2 C_{ii}^\agi(t)} {\delta}}-\sqrt{\frac{2 C_{ii}^\agii(t)} {\delta}}.$

Of course, this argument can be made for any choice of dimension $i$. Hence, a less conservative, yet valid, choice is
\begin{equation}
\label{eq:critfctChebyshev_1dim_infiniteradii_better}
\gamma^{\agi,\agii} (t) := \max_{i=1,...,D} \abs{\mu^\agi_i(t) - \mu_i^\agii(t) } -\Lambda^{\agi,\agii} - \sqrt{\frac{2 C_{ii}^\agi(t)} {\delta}}-\sqrt{\frac{2 C_{ii}^\agii(t)} {\delta}}.
\end{equation}

Notice, this function has the desirable property of being Lipschitz continuous, provided the mean $\mu_i^\agiii : I \to \Real$ and standard deviation functions $\sigma_{ii}^\agiii= \sqrt{C^\agiii_{ii}} : I \to \Real_+$ are. In particular, it is easy to show $L(\gamma^{\agi,\agii}) \leq  \max_{i=1,...,D} L( \mu_i^\agi ) + L(\mu_i^\agii) + \sqrt{\frac{2}{\delta}} \bigl(L(\sigma_{ii}^\agi) + L(\sigma_{ii}^\agii) \bigr)$
where, as before, $L(f)$ denotes a Lipschitz constant of function $f$.

For the special case of two dimensions, we can derive a less conservative alternative criterion function based on a tighter two-dimensional Chebyshev-type bound \cite{whittle_chebyshev}:

\begin{thm}[Alternative collision criterion function] \label{def:collcritfct2d}
 Let spatial dimensionality be $D = 2$. Choosing
 
$r^{\agiii}_i(t) := \sqrt{\frac{1}{2\delta^a}}\, \sqrt{ C_{ii}^\agiii(t) +\frac{ \sqrt{C_{ii}^\agiii(t) C_{jj}^\agiii(t) 
(C_{ii}^\agiii(t) C_{jj}^\agiii(t) - (C_{ij}^\agiii(t))^2) }}{C_{jj}^\agiii(t)}}$
 
($\agiii \in \{\agi,\agii\}, i \in \{1,2\}, j \in \{1,2\} - \{i\}$) in Eq. \ref{eq:critfctgeneric} yields a valid distribution-independend criterion function. That is, $\gamma^{\agi,\agii}(t) >0 \Rightarrow \Pr [\mathfrak C^{\agi,\agii}(t)]  < \delta^\agi$.
\end{thm}

A proof sketch and a Lipschitz constant (for non-zero uncertainty) are provided in the appendix. Note, the Lipschitz constant we have derived therein becomes infinite in the limit of vanishing variance. In that case, the presence of negative criterion values can be tested based on the sign of the minimum of the criterion function. This can be found employing a global optimiser. Future work will investigate, in how far Hoelder continuity instead of Lipschitz continuity can be leveraged to yield a similar algorithm as the one provided in Sec. \ref{sec:adaptiveLipshubertstyle}.

\subsubsection{Multi-agent case.} Let $\agi \in \agset$, $\agset' \subset \agset$ such that $\agi \notin \agset'$ a subset of agents. We define the event that $\agi$ collides with at least one of the agents in \agset' at time $t$ as  $\mathfrak C^{\agi,\agset'}(t) := \{ (\state^\agi(t),\state^\agii(t)) | \exists \agii \in \agset': \norm{\state^\agi(t)-\state^\agii(t)}_2 \leq \Lambda  \} = \bigcup_{\agii \in \agset'} \mathfrak C^{\agi,\agii}$.
By union bound, $\Pr[ \mathfrak C^{\agi,\agset'}(t)] \leq \sum_{\agii \in \agset'} \Pr[ \mathfrak C^{\agi,\agii}(t)] $.

\begin{thm} [Multi-Agent Criterion] Let $\gamma^{\agi,\agii}$ be valid criterion functions defined w.r.t. collision bound $\delta^\agi$.
We define  \emph{multi-agent collision criterion function} $\Gamma^{\agi,\agset'}(t) := \min_{\agii \in \agset'} \gamma^{\agi,\agii}(t)$. If  $\Gamma^{\agi,\agset'}(t) > 0$ then the collision probability with \agset' is bounded below $\delta^\agi |\agset'|$. That is, $\Pr[ \mathfrak C^{\agi,\agset'}(t)] < \delta^\agi |\agset'|.$ 
\label{thm:mascritfct}
\end{thm}
\begin{proof}

Let $\agi \in \agset$, $\agset' \subset \agset$ such that $\agi \notin \agset'$ a subset of agents. We define the event that $\agi$ collides with at least one of the agents in \agset' at time $t$ as  $\mathfrak C^{\agi,\agset'}(t) := \{ (\state^\agi(t),\state^\agii(t)) | \exists \agii \in \agset': \norm{\state^\agi(t)-\state^\agii(t)}_2 \leq \Delta  \} = \bigcup_{\agii \in \agset'} \mathfrak C^{\agi,\agii}$.

We have established that if $\forall \agii \in \agset': \gamma^{\agi,\agii}(t) >0$ then $\Pr [\mathfrak C^{\agi,\agii}(t)]  < \delta^\agi,\forall \agii \in \agset' $.  Now, let  $\Gamma^{\agi, \agset' }< \delta^\agi$. Hence,$\forall \agii \in \agset': \gamma^{\agi,\agii}(t) >0$. Thus, $\forall \agii \in \agset': \Pr[ \mathfrak C^{\agi,\agii}(t)]  < \delta^\agi)$ Therefore, $\sum_{\agii \in \agset'} \Pr[ \mathfrak C^{\agi,\agii}(t)] \leq \abs{\agset'} \delta^\agi$. By union bound, $\Pr[ \mathfrak C^{\agi,\agset'}(t)] \leq \sum_{\agii \in \agset'} \Pr[ \mathfrak C^{\agi,\agii}(t)] $. Consequently, we have $\Pr[ \mathfrak C^{\agi,\agset'}(t)] \leq \abs{\agset'} \delta^\agi$. q.e.d.

\end{proof}

Moreover, $\Gamma^{\agi,\agset'}$ is Lipschitz if the constituent functions $\gamma^{\agi,\agii}$ are (see Appendix \ref{sec:derlipno}).

\begin{figure*}
        \centering
        \begin{subfigure}
                \centering
                \includegraphics[width = 3.9cm, clip, trim = 3cm 9.5cm 4cm 10cm]
								{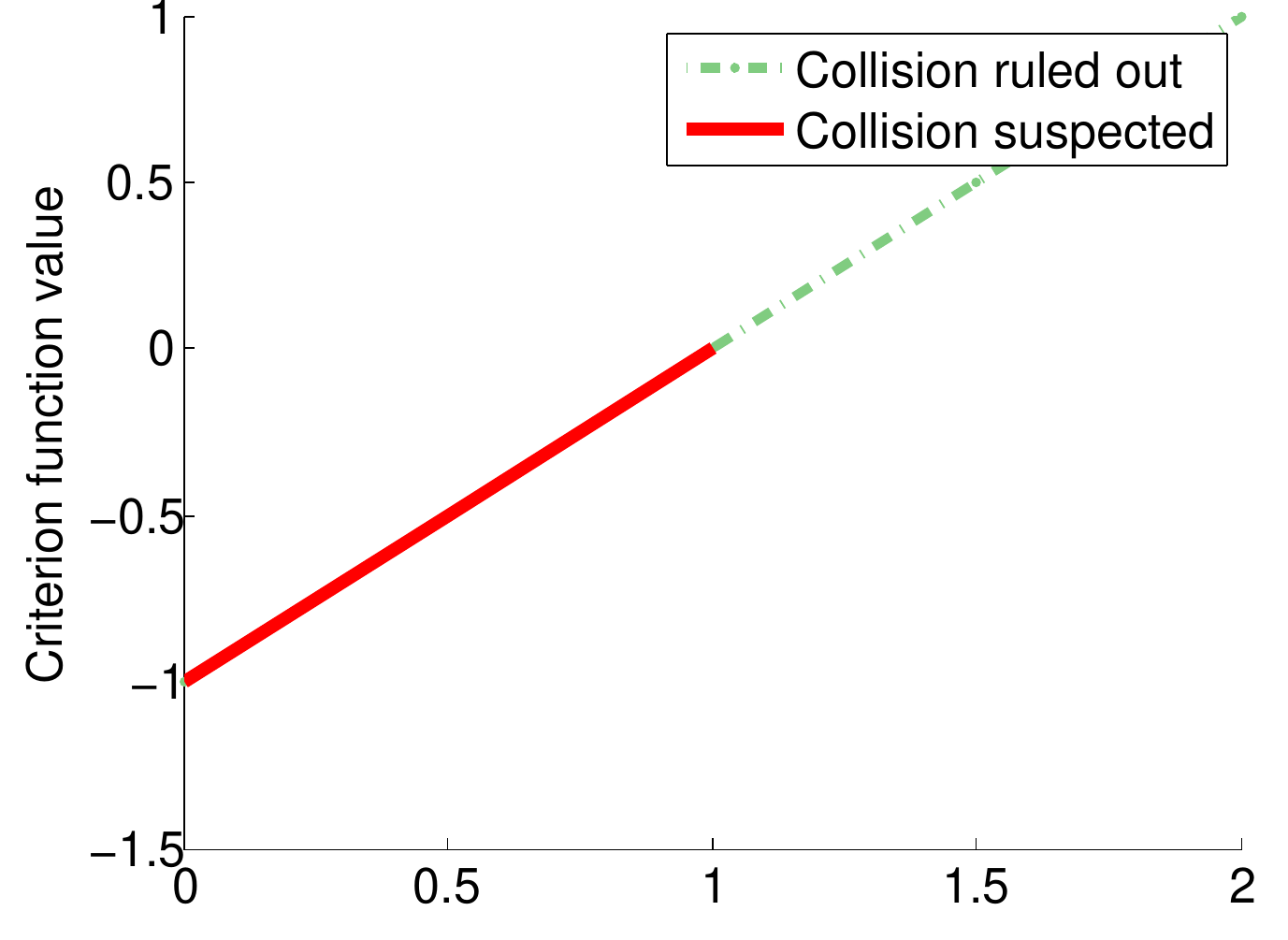}
        \end{subfigure}%
        \begin{subfigure}
                \centering
                \includegraphics[width = 3.9cm, clip, trim = 3cm 9.5cm 4cm 10cm]
								{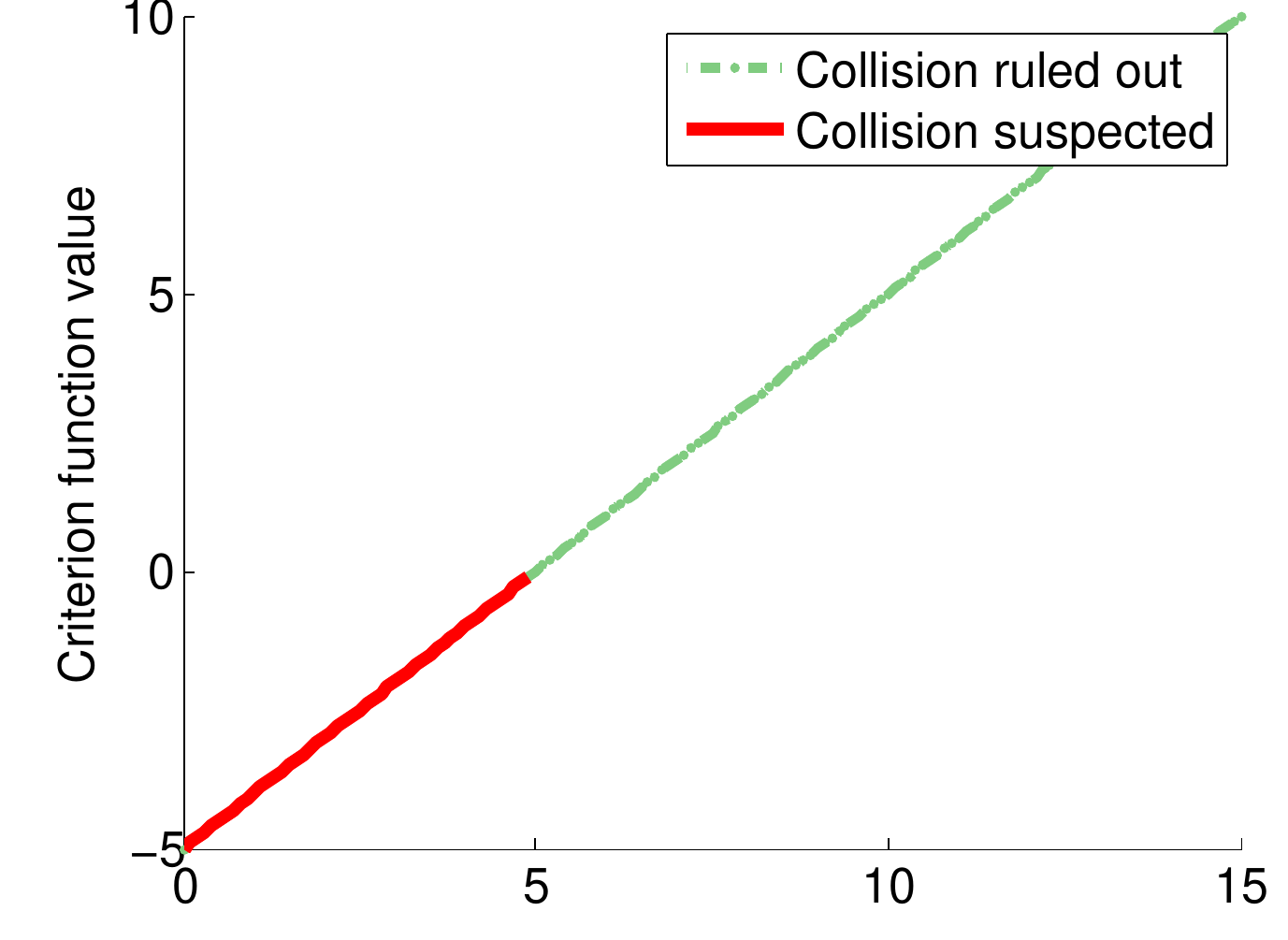}
        \end{subfigure}
        \begin{subfigure}
                \centering
                \includegraphics[width = 3.9cm, clip, trim = 3cm 9.5cm 4cm 10cm]
								{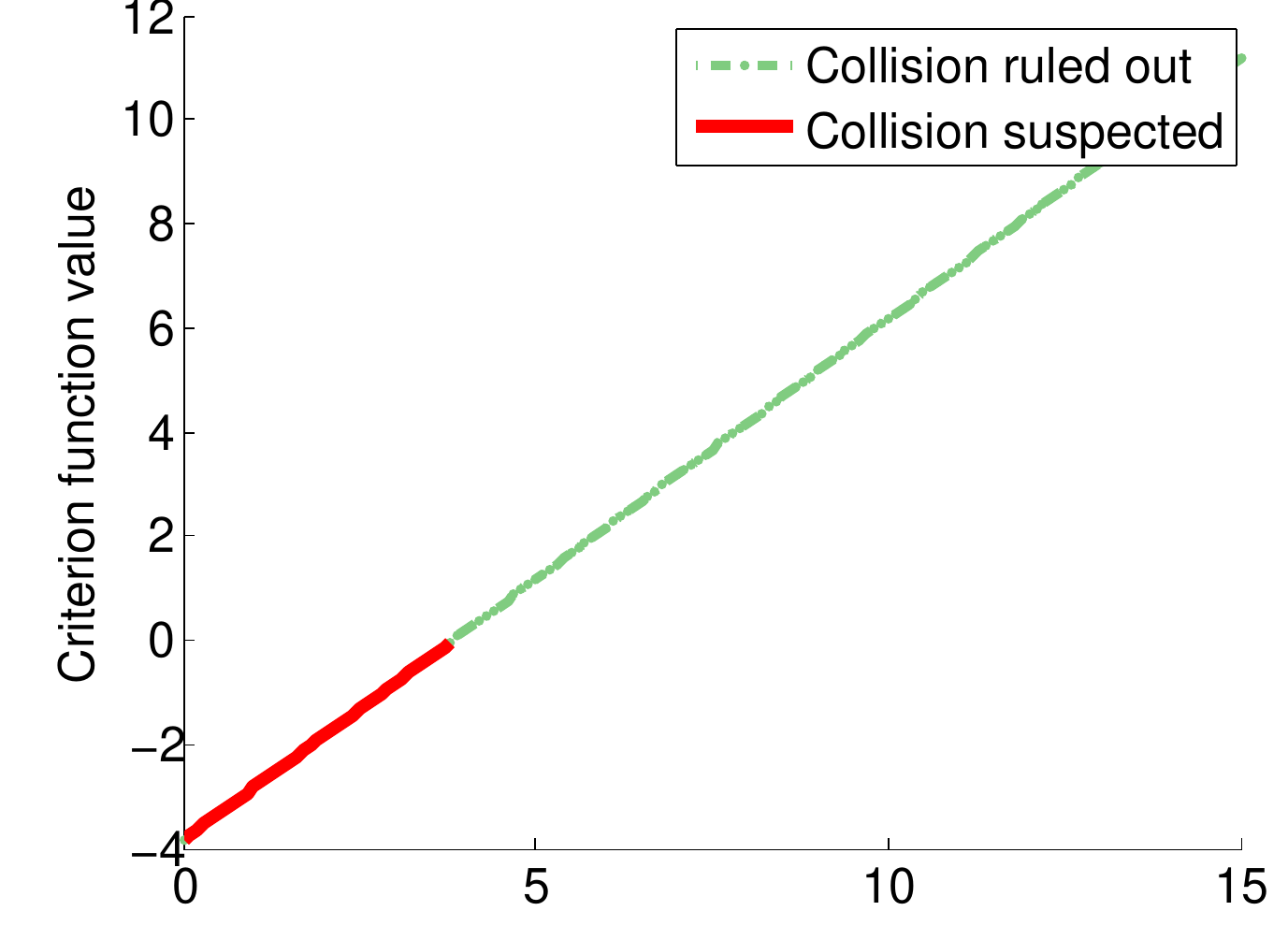}
        \end{subfigure}
       \caption{Criterion function values (as per Eq. \ref{eq:critfctChebyshev_1dim_infiniteradii_better}) as a function of $\norm{\expect{\state^\agii} - \expect{\state^\agi}}_\infty$  and with $\delta =0.05$, $\Lambda^{\agi,\agii} =1$. Left: variances $C^\agi = C^\agii = \text{diag}(.00001,.00001)$. Centre: variances $C^\agi = C^\agii = \text{diag}(.1,.1)$. Right: variances $C^\agi = C^\agii = \text{diag}(.1,.1)$ and with improved criterion function (as per Thm. \ref{def:collcritfct2d}).
	} 
       \label{fig:whittlecritfct}
\end{figure*}

Our distribution-independent collision criterion functions have the virtue that they work for all distributions -- not only the omnipresent Gaussian. Unfortunately, distribution-independence is gained at the price of conservativeness ( ref. to Fig. \ref{fig:whittlecritfct}). In our experiments in Sec. \ref{sec:sims}, the collision criterion function as per Thm. \ref{def:collcritfct2d} is utilized as an integral component of our collision avoidance mechanisms. The results suggest that the conservativeness of our detection module does not entail prohibitively high-false-alarm rates for the distribution-independent approach to be considered impractical. That said, whenever distributional knowledge can be converted into a criterion function. One could then use our derivations as a template to generate refined criterion functions using Eq. \ref{eq:critfctgeneric} with adjusted radii $r_i$,$r_j$, reflecting the distribution at hand.

\section{Collision Avoidance} 

In this section we outline the core ideas of our proposed approach to multi-agent collision avoidance.
After specifying the agent's dynamics and formalizing the notion of a single-agent plan, we define the multi-agent planning task. Then we describe how conflicts, picked-up by our collision prediction method, can be resolved. In Sec. \ref{sec:coord} we describe the two coordination approaches we consider utilizing to generate conflict-free plans. 

\textbf{I) Model (example).} We assume the system contains a set  $\agset$ of agents indexed by $ \agi \in \{1,...,| \agset\ | \}$. Each agent \agi's associated plant has a probabilistic state trajectory following stochastic controlled $D$-dimensional state dynamics (we consider the case $D=2$) in the continuous interval of (future) time $I=(t_0,t_f]$. We desire to ask agents to adjust their policies to avoid collisions. Each policy gives rise to a stochastic belief over the trajectory resulting from executing the policy.
For our method to work, all we require is that the trajectory's mean function $m:I \to \Real^D$ and covariance matrix function $\Sigma : I \to \Real^{D \times D}$ are evaluable for all times $t \in I$. 

A prominent class for which closed-form moments can be easily derived are linear stochastic differential equations (\textit{SDE}s). For instance, we consider the SDE 
\begin{equation}
\label{eq:linSDEcontrolledplant1}
d\state^\agi(t) = K \bigl(\xi^\agi(t) - \state^\agi(t)\bigr) dt + B \, dW 
\end{equation} where $K, B \in \Real^{D \times D}$ are matrices $\state^\agi: I \to \Real^D$ is the state trajectory and $W$ is a vector-valued Wiener process. Here, $u(\state^\agi; \xi^\agi) :=K( \xi^\agi - \state^\agi )$ could be interpreted as the control policy of a linear feedback-controller parametrised by $\xi^\agi$. It regulates the state to track a desired trajectory $\xi^\agi(t)= \zeta_0^\agi  \chi_{\{0\}} (t)+\sum_{i=1}^{H^\agi} \zeta_i^\agi  \chi_{\tau_i^\agi} (t)$ where $\chi_{\tau_i}: \Real \to \{0,1\}$ denotes the indicator function of the half-open interval $\tau_{i}^\agi = (t_{i-1}^\agi, t_{i}^\agi] \subset [0,T^\agi]$ and each $\zeta^\agi_i \in \Real^D$
is a \textit{setpoint}. If $K$ is positive definite the agent's state trajectory is determined by setpoint sequence $p^\agi = (t^a_i,\zeta^\agi_i)_{i=0}^{H^\agi}$ (aside from the random disturbances) which we will refer to as the agent's \emph{plan}.
For example, plan $p^\agi:= \bigl( (t_0,\state_0^\agi ), (t_f,\state_f^\agi ) \bigr)$ could be used to regulate agent $\agi$'s \textit{start state} $\state_0^\agi$ to a given \emph{goal state} $\state_f^\agi$ between times $t_0$ and $t_f$. For simplicity, we assume the agents are always initialized with plans of this form before coordination commences.

One may interpret a setpoint as some way to alter the stochastic trajectory. Below, we will determine setpoints that modify a stochastic trajectory to reduce collision probability while maintaining low expected cost. From the vantage point of policy search, $\xi^\agi$ is agent $\agi$'s policy parameter that has to be adjusted to avoid collisions.

\textbf{II) Task.} Each agent $\agi$ desires to find a sequence of setpoints $(p^\agi)$ such that (i) it moves from its start state $x_0^\agi$ to its goal state $x_f^\agi$ along a low-cost trajectory and (ii) such that along the trajectory its plant (with diameter $\Delta$) does not collide with any other agents' plant in state space with at least a given probability $1-\delta \in (0,1)$.

\textbf{III) Collision resolution}. An agent seeks to avoid collisions by adding new setpoints to its plan until the collision probability of 
the resulting state trajectory drops below threshold $\delta$. 
For choosing these new setpoints we consider two methods \textbf{WAIT} and \textbf{FREE}.
In the first method the agents insert a time-setpoint pair $(t,x_0^\agi)$ into the previous plan $p^\agi$. 
Since this aims to cause the agent to wait at its start location $x^\agi_0$ we will call the method {\scshape WAIT}. It is possible that 
multiple such insertions are necessary until collisions are avoided. Of course, if a higher-priority agent decides to traverse through $x_0^\agi$, 
this method is too rigid to resolve a conflict.
In the second method the agent optimizes for the time and location of the new setpoint. Let $p^\agi_{\uparrow(t,s)} $ be the 
plan updated by insertion of time-setpoint pair $(t,s) \in I \times \Real^D$.   We propose to choose the candidate setpoint $(t,s)$ that 
minimizes a function being a weighted sum of the expected cost entailed by executing updated plan $p^\agi_{\uparrow(t,s)}$ and a hinge-loss collision penalty 
$c_{coll}^\agi(p^\agi_{\uparrow(t,s)}) :=   \lambda  \,\max\{0, -\min_t \Gamma^\agi(t)\} $. Here, $\Gamma^\agi$ 
is computed based on the assumption we were to execute $p^\agi_{\uparrow(t,s)} $ and $\lambda >>0$ determines the extent to which collisions are penalized. Since the new setpoint can be chosen freely in time and 
state-space we refer to the method as {\scshape FREE}.

\subsection{Coordination} \label{sec:coord}
We will now consider how to integrate our collision detection and avoidance methods into a coordination framework that determines who needs to avoid whom and at what stage of the coordination process. Such decisions are known to significantly impact the \textit{social cost} (i.e. the sum of all agents' individual costs) of the agent collective.\\

\textbf{Fixed-priorities (FP).}
As a baseline method for coordination we consider a basic fixed-priority method (e.g. \cite{erdmann_movingobj:87,Bennewitz01Exploiting}).
Here, each agent has a unique ranking (or priority) according to its index $\agi$ (i.e. agent 1 has highest priority, agent $|\agset|$ lowest). When all higher-ranking agents are done planning, agent $\agi$ is informed of their planned trajectories which it has to avoid with a probability greater than $	1-\delta$. This can be done by repeatedly invoking for collision detection and resolution methods described above until no further collision with higher-ranking agents are found. 

\textbf{Lazy Auction Protocol (AUC).}
While the FP method is simple and fast the rigidity of the fixed ranking can lead to sub-optimal social cost and coordination success. Furthermore, its sequential nature does not take advantage of possible parallelization a distributed method could. To alleviate this we propose to revert the ranking flexibly on a case-by-case basis. In particular, the agents are allowed to compete for the right to gain passage (e.g. across a region where a collision was detected) by submitting bids in the course of an auction.  The structure of the approach is outlined in Alg. \ref{alg:lazyauctions}.

\IncMargin{-1em}
\begin{algorithm}
\begin{small}
\SetKwData{Left}{left}\SetKwData{This}{this}\SetKwData{Up}{up} \SetKwData{Constraints}{Constraints} 
 \SetKwData{Collisions}{Collisions} \SetKwData{flag}{flag} \SetKwData{C}{$\mathcal C$} \SetKwData{winner}{winner} \SetKwData{tcoll}{$t_{\text{coll}}$}
\SetKwFunction{Union}{Union}\SetKwFunction{FindCompress}{FindCompress} \SetKwFunction{Resolve}{Resolve}
\SetKwFunction{Broadcast}{Broadcast} \SetKwFunction{Planner}{Planner} \SetKwFunction{Auction}{Auction}
\SetKwFunction{Receive}{Receive} \SetKwFunction{Avoid}{Avoid}
\SetKwFunction{DetectCollisions}{CollDetect}
\SetKwInOut{Input}{input}\SetKwInOut{Output}{output}
\Input{Agents $\agi \in \agset$, cost functions $c^\agi$, dynamics, initial start and goal states, initial plans $p^1,...,p^{|\agset|}$ .}
\Output{collision-free plans $p^1,...,p^{|\agset|}$.}
\BlankLine
\Repeat{$\forall \agi \in \agset: \flag^\agi =0 $}
{
\For{$\agi \in \agset$}{[
\flag$^\agi,\C^\agi,\tcoll ]\leftarrow$ \DetectCollisions$^\agi (\agi, \agset-\{\agi\})$\\

\If{$\flag^\agi =1$}
{
$\winner \leftarrow \Auction(\C^\agi \cup \{\agi\},\tcoll )$\\
\ForEach{$\agii \in (\C^\agi \cup \{\agi\})-\{\winner\} $}{
$p^\agii \leftarrow \Avoid^\agii((\C^\agi \cup \{\agi\})-\{\agii\},\tcoll )$\\
$\Broadcast^\agii$ ($p^\agii$)
}
}

}

}
\caption{Lazy auction coordination method (AUC) (written in a sequentialized form). Collisions are resolved by choosing new setpoints to enforce collision avoidance. 
$\mathcal C^\agi$: set of agents detected to be in conflict with agent $\agi$. flag$^\agi$: collision detection flag (=0, iff no collision detected). $t_{\text{coll}}$: earliest time where a collision was detected. Avoid: collision resolution method updating the plan by a single new setpoint according to WAIT or FREE.}

\label{alg:lazyauctions}
\end{small}
\end{algorithm}

Assume an agent $\agi$ detects a collision at a particular time step $t_{\text{coll}}$ and invites the set of agents $\mathcal C^\agi = \{ \agii | \gamma^{\agi,\agii} (t_{\text{coll}}) \leq 0\}$ to join an auction to decide who needs to avoid whom. In particular, the auction determines a winner who is not required to alter his plan. The losing agents need to insert a new setpoint into their respective plans designed to avoid all other agents in $\mathcal C^\agi$ while keeping the plan cost function low. 

The idea is to design the auction rules as a heuristic method to minimize the social cost of the ensuing solution. To this end, we define the bids such that their magnitude is proportional to a heuristic magnitude of the expected regret for losing and not gaining passage. That is agent $\agi$ submits a bid $b^\agi = \mathfrak l^\agi - \mathfrak s^\agi$. Magnitude $\mathfrak l^\agi$ is defined as $\agi$'s anticipated cost $c_{plan}^\agi(p^\agi_{\uparrow(t,s)})$ for the event that the agent will not secure ``the right of passage'' and has to create a new setpoint $(t,s)$ (according to (III)) tailored to avoid all other agents engaged in the current auction. On the other hand, $\mathfrak s^\agi := c_{plan}^\agi(p^\agi)$ is the cost of the unchanged plan $p^\agi$.
If there is a tie among multiple agents the agent with the lowest index among the highest bidders wins.

Acknowledging that $\mathfrak s^{winner} + \sum_{\agi\neq winner} \mathfrak l^\agi $
is an estimated social cost (based on current beliefs of trajectories) 
after the auction, we see that the winner determination
rule greedily attempts to minimize social cost: $b^{winner}
\geq b^{\agii}  \,  \Leftrightarrow \forall \agii: \mathfrak s^\agii  + \sum_{\agi\neq
\agii} \mathfrak l^{ \agi} \geq \mathfrak s^{winner} + \sum_{\agi\neq winner} \mathfrak l^{\agi}$.


\begin{figure*}[t!] 
\vspace{-1em}
\begin{tabular}{lll}
 \includegraphics[width = 3.5cm, clip, trim = 3.5cm 9.5cm 4.5cm 9cm]{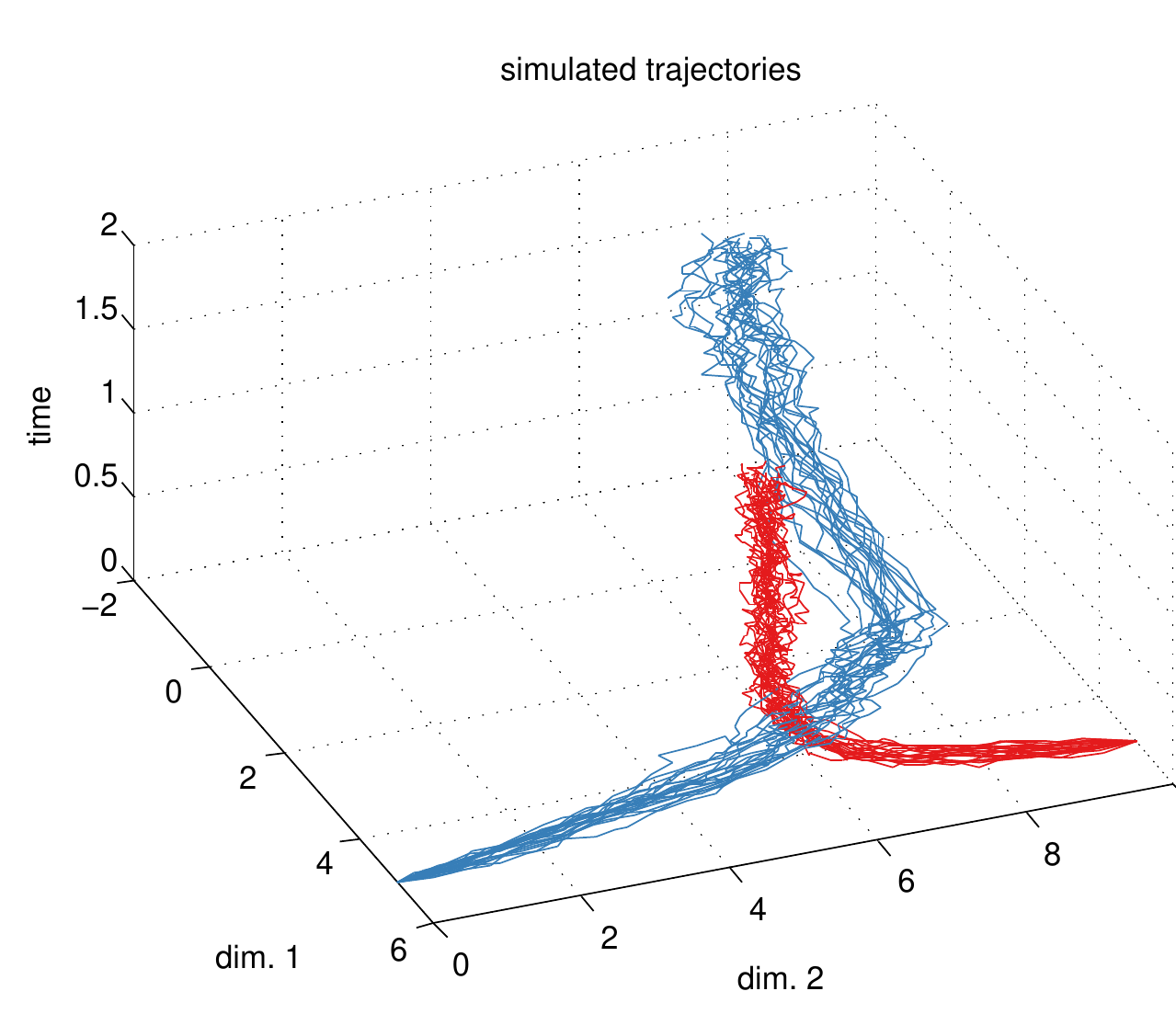}
    & \includegraphics[width = 3.5cm, clip, trim = 3.5cm 9.5cm 4.5cm 9cm]{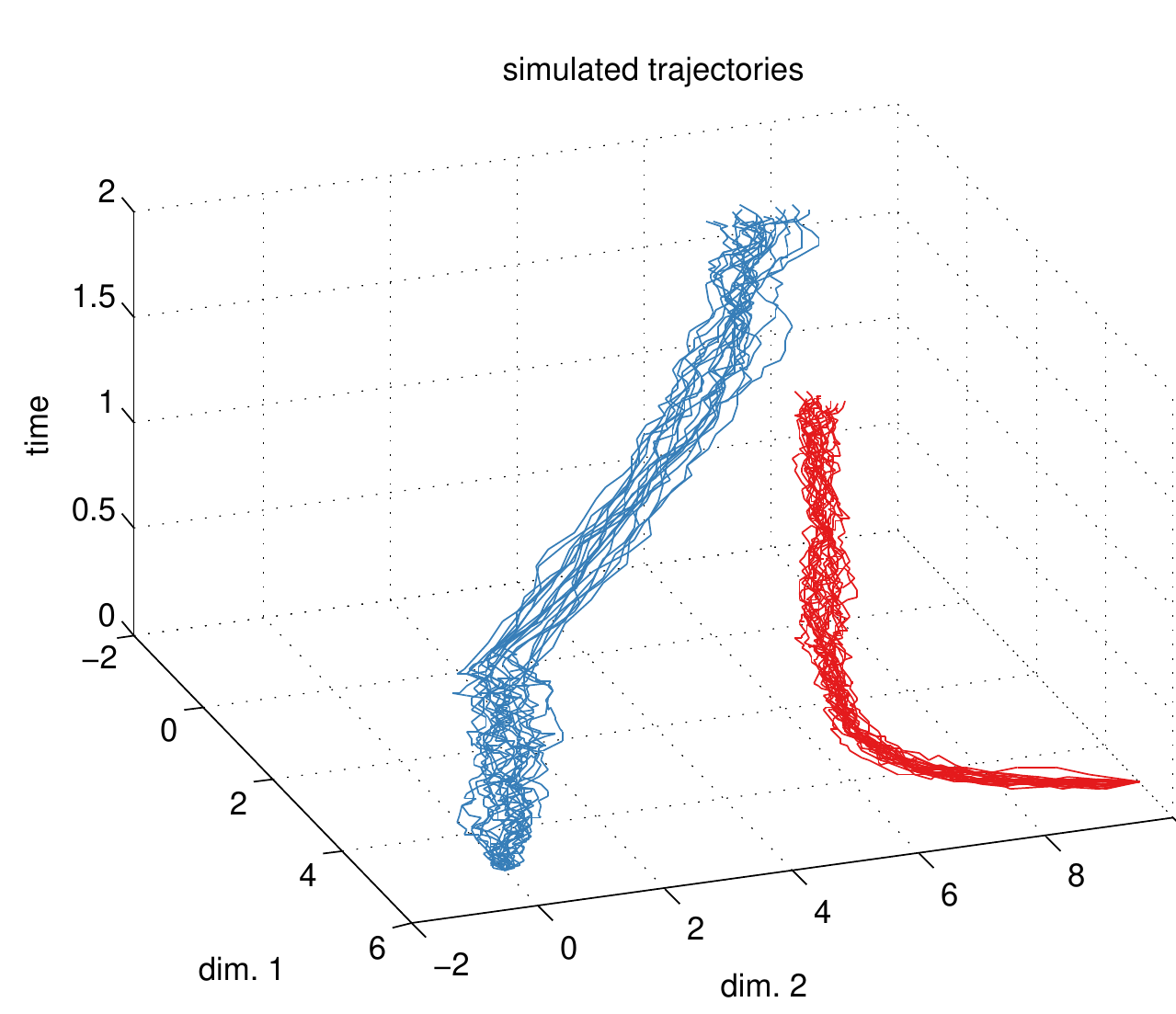} 
&\includegraphics[width = 3.5cm, clip, trim = 3.5cm 9.5cm 6cm 9cm]{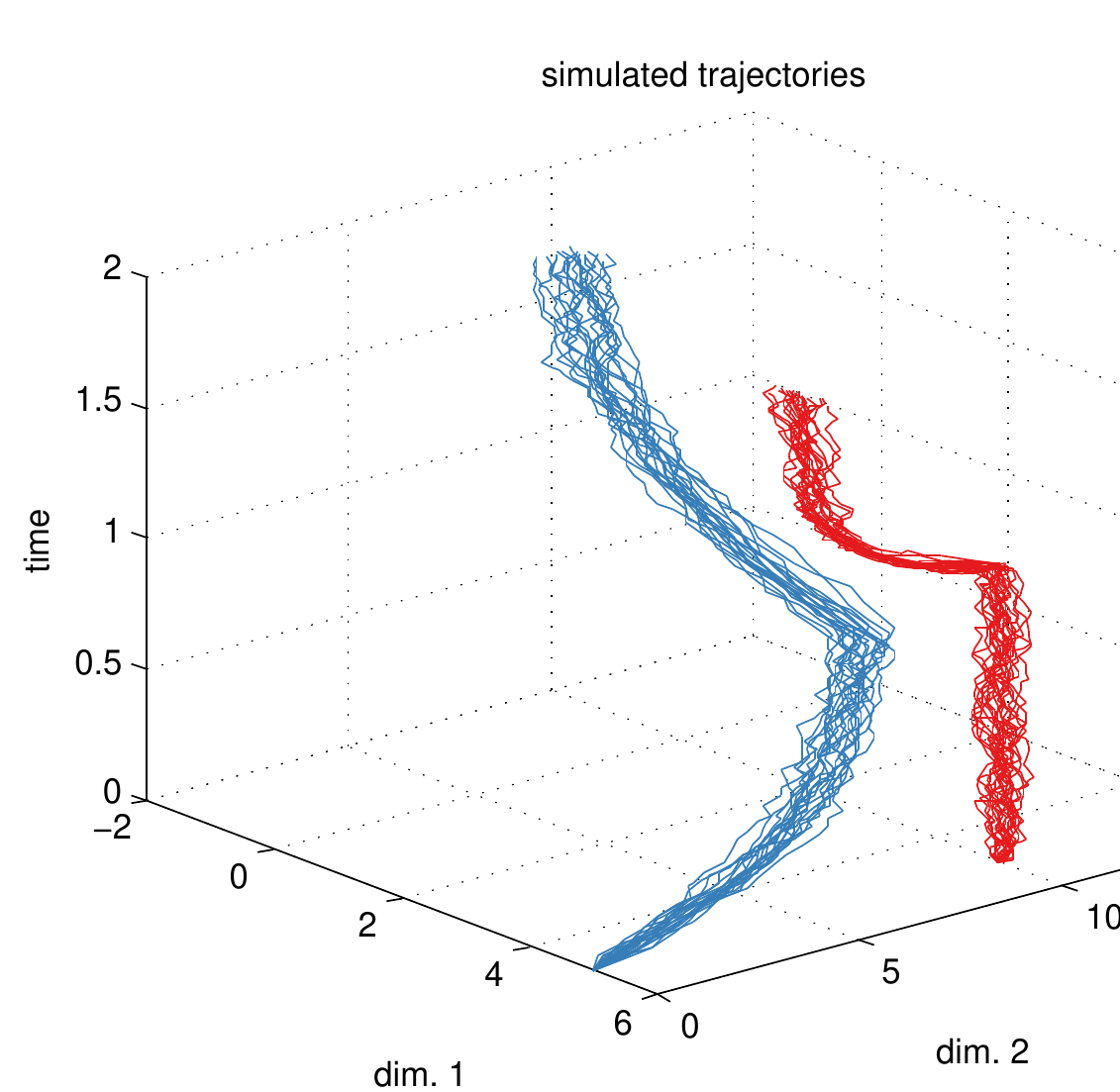} \\
\end{tabular}
\caption{EXP1. Draws from uncoordinated agents' plans (left), after coordination and collision resolution with methods {\scshape FP-WAIT} (centre) and {\scshape AUC-WAIT} (right).}
\label{Tab:EXP1corridor}
\end{figure*} 


 \begin{table*}[bht]   
 \centering
 \begin{small} 
 \begin{tabular}{@{}lcccccc@{}}
 \toprule
 & \multicolumn{3}{c}{Experiment $1$} & 
 \multicolumn{3}{c}{Experiment $2$} \\
 \cmidrule(r){2-4} \cmidrule(l){5-7} 
$Quantity$ & NONE &  AUC-WAIT& FP-WAIT   &NONE& AUC-FREE &FP-FREE  \\
 \midrule
  A		& 78 	& 0 & 0 	& 51 	& 0 & 0 \\
{B} & 13.15 	& 13.57 & 12.57 	& 14.94 	& 16.22 & 18.13\\
{C} & 0.05 	& 0.04 & 25.8	& 0.05 	& 0.05 & 0.05 \\
{D}& 0	& 6 & 3		& 0 	& 4	& 4\\
 \bottomrule
 \end{tabular}
 \end{small}
 \caption{Quantities estimated based on 100 draws from SDEs simulating executions of the different plans in EXP1 and EXP2.
\textbf{A}: estimated collision probability $[\%]$; \textbf{B}: averaged path length away from goal; \textbf{C}: averaged sqr. dist. of final state to goal; \textbf{D}: number of collision resolution rounds.
 Notice, our collision avoidance methods succeed in preventing collisions. In EXP1 the FP-WAIT method failed to reach its first goal in time which is reflected in the \emph{sqr. distance to goal} measure. Note the discrepancies in avg. path length are relatively low due to convexity effects and the contribution of state noise to the path lengths.}
 \label{Tab:data}
 \end{table*}


\begin{figure*}[thb!] 
\vspace{-1em}
\begin{tabular}{lll}
 \includegraphics[width = 3.9cm, clip, trim = 3.5cm 9cm 4.5cm 9cm]{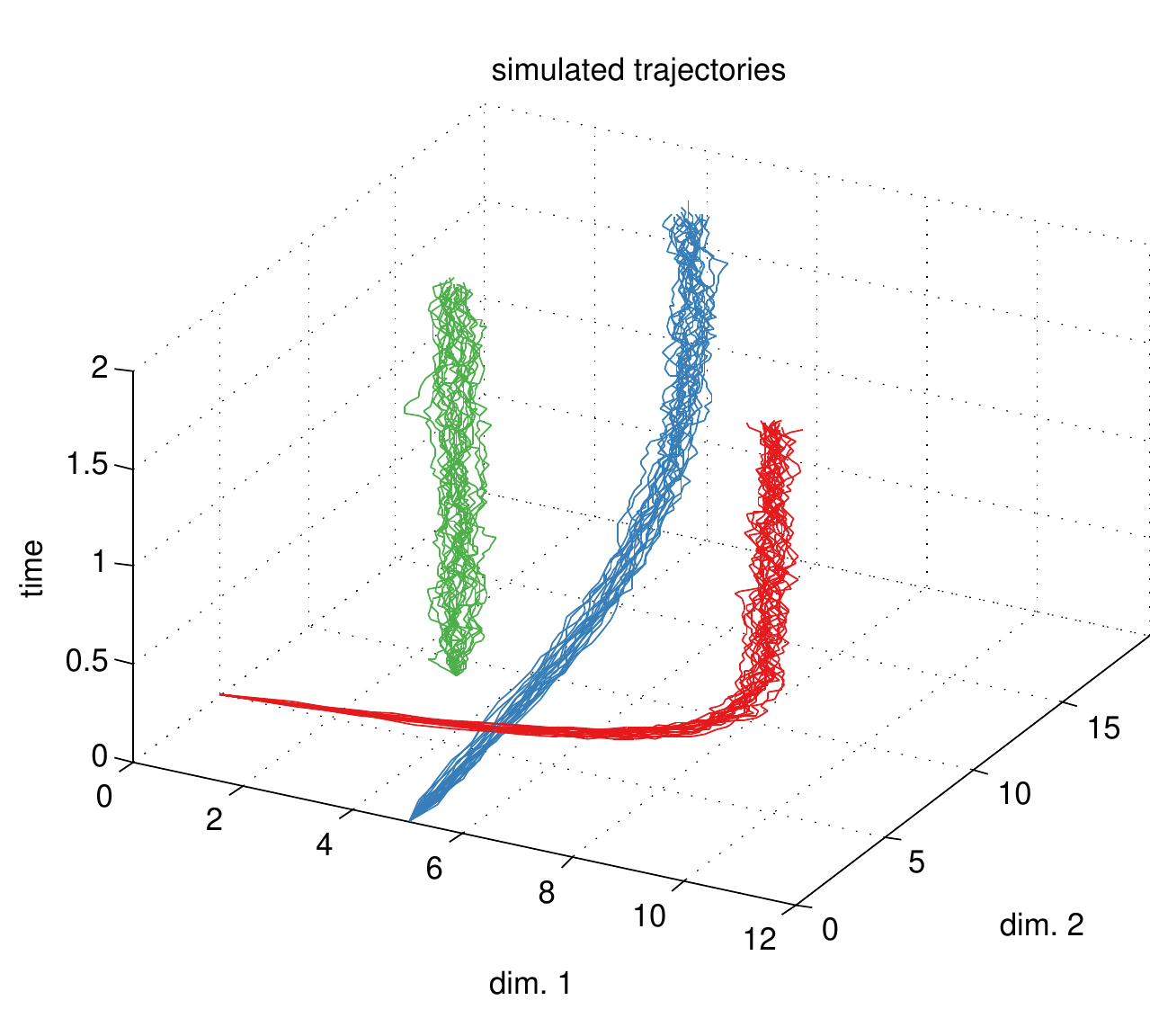}
    & \includegraphics[width = 3.9cm, clip, trim = 3.5cm 9cm 4.5cm 9cm]{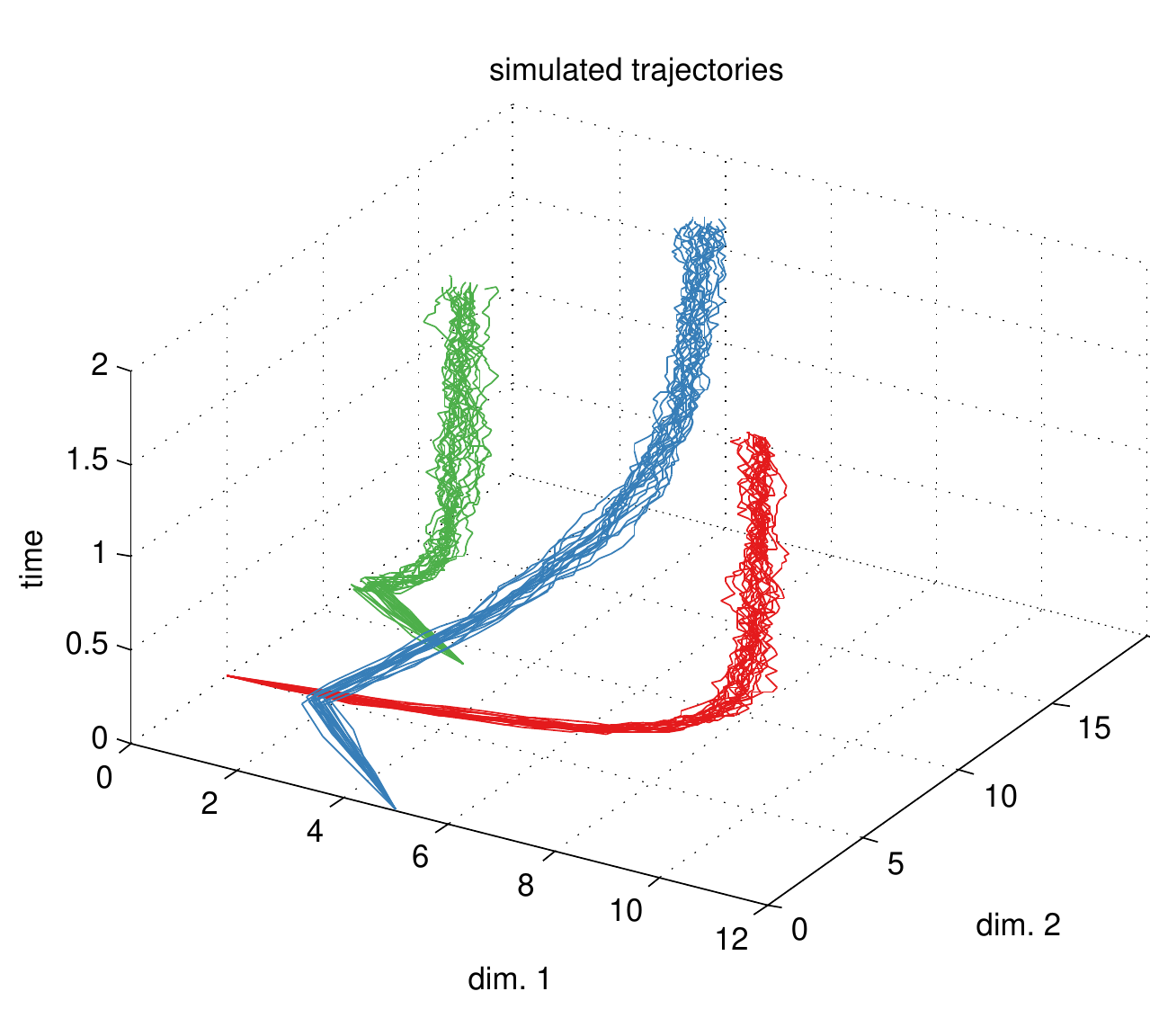} 
&\includegraphics[width = 3.9cm, clip, trim = 3.5cm 9cm 4.5cm 9cm]{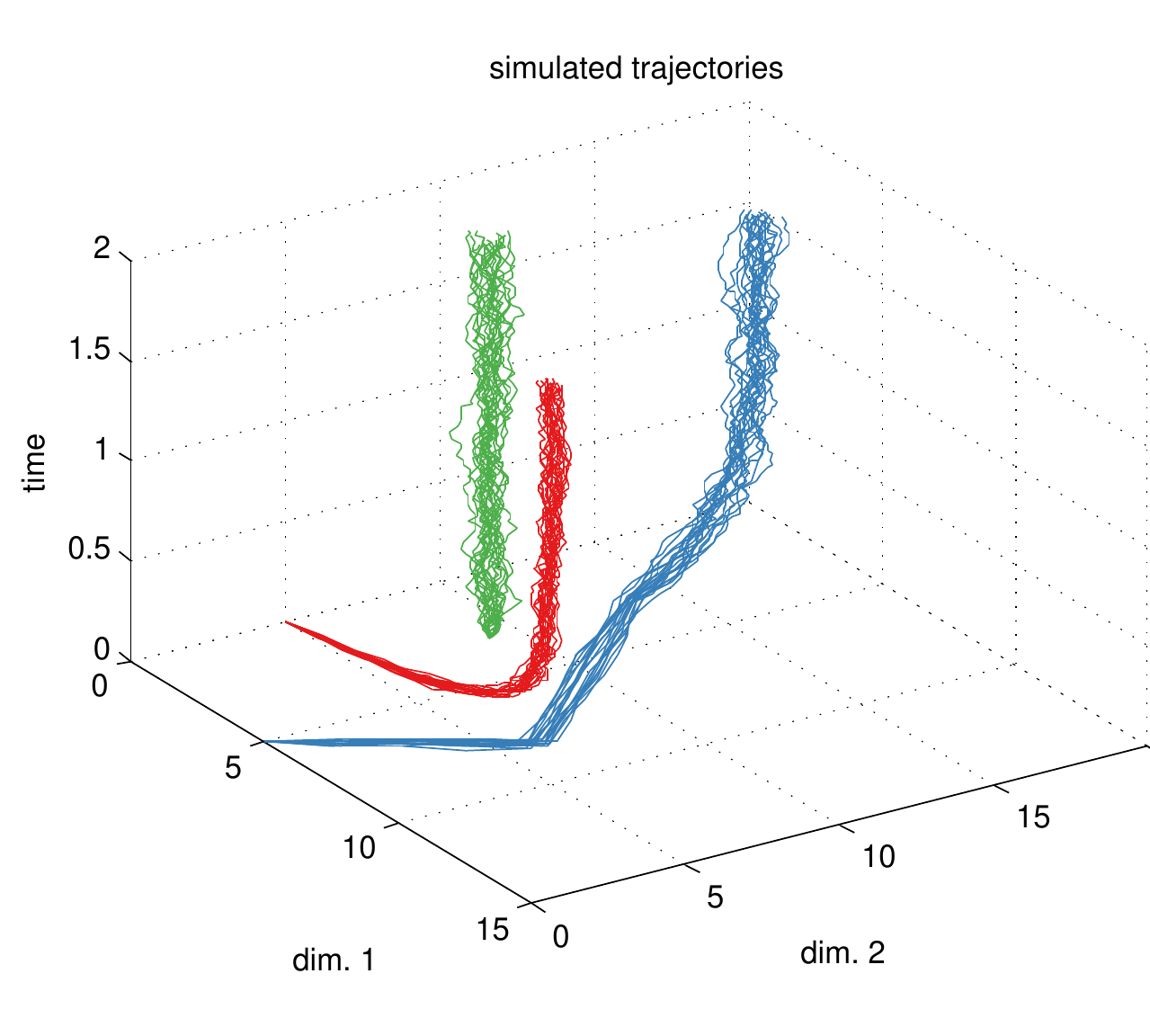} \\    
\end{tabular}
\caption{EXP2. Draws from uncoordinated agents' plans (left), after coordination and collision resolution with methods {\scshape FP-FREE} (centre) and {\scshape AUC-FREE} (right). }
\label{Tab:EXP2}
\end{figure*} 

%

\section{Simulations}\label{sec:sims}  As a first test, we simulated three simple multi-agent scenarios, \textit{EXP1}, \textit{EXP2} and \textit{EXP3}.
Each agent's dynamics were an instantiation of an SDE of the form of Eq. \ref{eq:linSDEcontrolledplant1}.
We set $\delta$ to achieve collision avoidance with certainty greater than $95 \%$. Collision prediction was based on the improved criterion function as per Thm. \ref{def:collcritfct2d}. During collision resolution with the FREE method
 each agent $\agi$ assessed a candidate 
plan $p^\agi$ according to cost function 
$c_{plan}^\agi(p^\agi) = w_1 \,  c_{traj}^\agi(p^\agi) + w_2 \,  c_{miss}^\agi(p^\agi) + w_3 \, c_{coll}^\agi(p^\agi) $.
Here $c_{traj}^\agi$ is a heuristic to penalize expected control energy or path length; in the second summand, 
$c_{miss}^\agi (p^\agi)= \norm{\state^\agi(t_f) - \state^\agi_f}^2$ penalizes expected deviation from the goal state; the third term
$c_{coll}^\agi(p^\agi)$ penalizes collisions (cf. III ). The weights are design parameters which we set to 
$w_1 = 10, w_2 = 10^3$ and $w_3 = 10^6$, emphasizing avoidance of mission failure and collisions. Note, if our method was to be deployed in 
a receding horizon fashion, the parameters could also be adapted online using standard learning techniques such as no-regret algorithms 
\cite{littlestone89weighted,Srinivas2010}.

\begin{figure*}[thb!]
\centering
\vspace{-1em}
\begin{tabular}{cc}
 \includegraphics[scale =.4, clip, trim = 3cm 9cm 3cm 9cm]{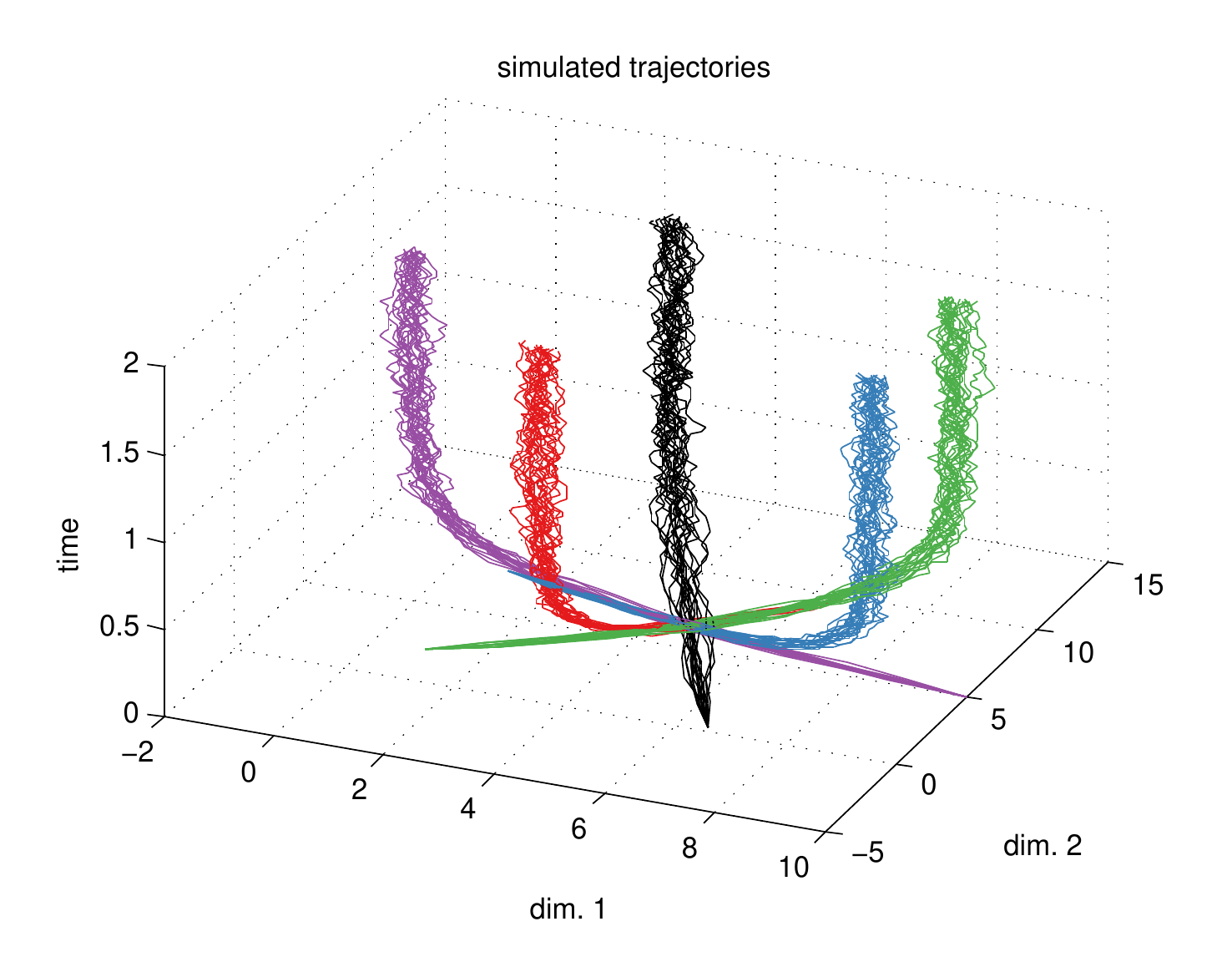}
    & \includegraphics[scale =.4, clip, trim = 3cm 9cm 3cm 9cm]{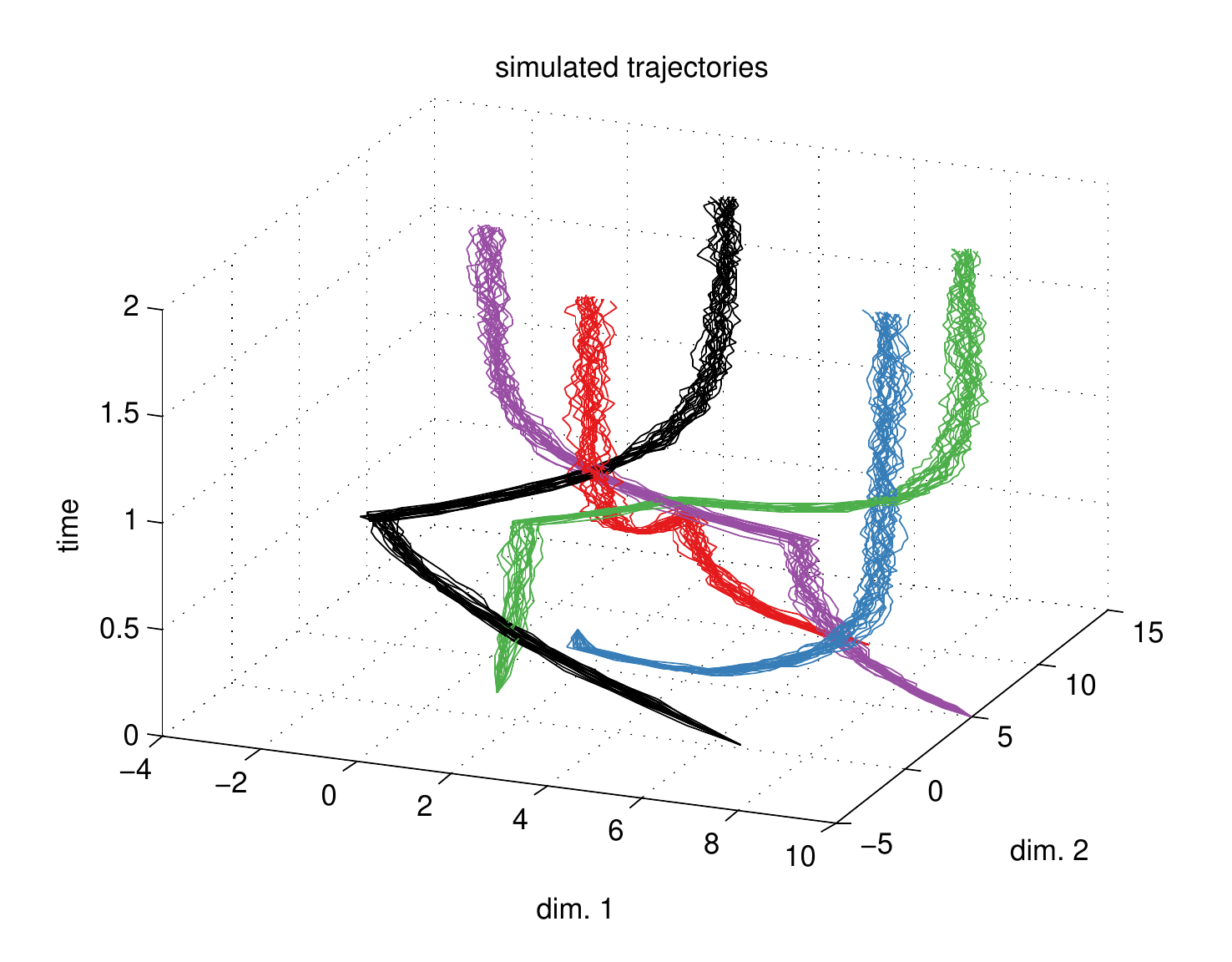} 
\end{tabular}
\caption{Ex. of EXP3 with 5 agents. Draws from uncoordinated agents' plans (left), after coordination and collision resolution with methods 
{\scshape AUC-FREE} (right). } \label{Tab:EXP35agents}
\end{figure*}


\begin{figure*}[thb!]
\vspace{-1em}
\begin{tabular}{ccc}
    \includegraphics[width = 4.6cm, clip, trim = 6cm 11.5cm 8cm 13cm]{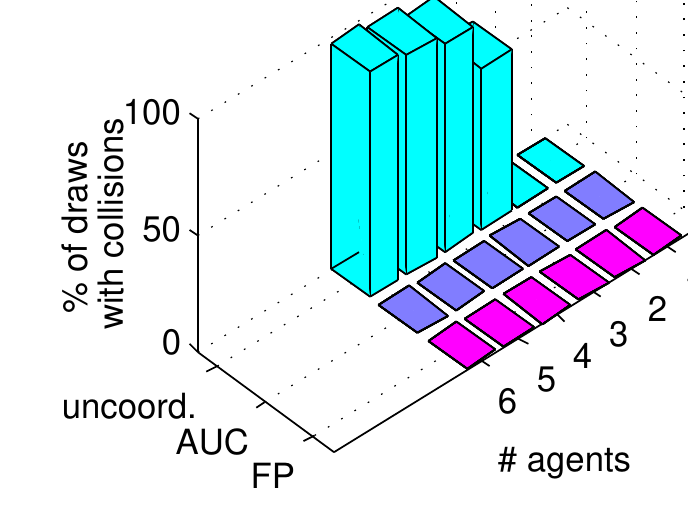} 
&\includegraphics[width = 4.6cm, clip, trim = 6cm 11.5cm 9cm 14cm]{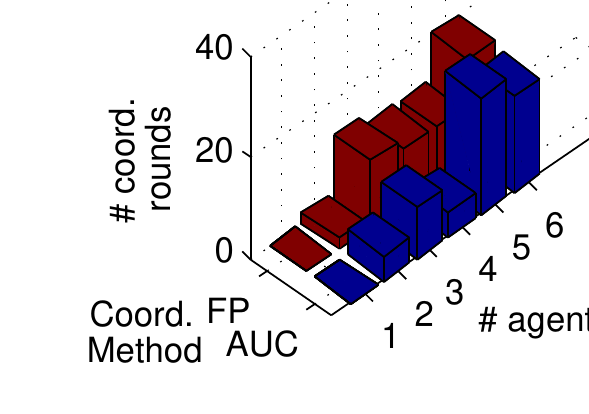} 
\end{tabular}
\caption{Recorded results for EXP3 with 1 to 6 agents. Note, all collisions were successfully avoided.} \label{Tab:exp3barplots}
\end{figure*} 

\textbf{EXP1.} Collision resolution was done with the WAIT method to update plans. Draws from the SDEs with the initial plans of the agents 
are depicted in Fig. \ref{Tab:EXP1corridor} (left). The curves represent 20 noisy trajectories of agents 1 (red) and 2 (blue). 
Each curve is a draw from the stochastic differential dynamics obtained by simulating the execution of the given initial plan. 
The trajectories were simulated with the Euler-Maruyama method for a time interval of $I = [0 s ,2 s]$. 
The spread of the families of curves is due to the random disturbances each agent's  controller had to compensate for during runtime.

Agent 1 desired to control the state from start state $\state_0^1 =(5,10)$ to goal $\state_f^1 =(5,5)$. 
Agent 2 desired to move from start state $\state_0^2 =(5,0)$ via intermediate goal $\state_{f_1}^2=(5,7)$ (at 1s) to final goal 
state $\state_{f_2}^2=(0,7)$. While the agents meet their goals under the initial plans, their execution would imply a high probability of 
colliding around state $(5,6)$ (cf. Fig. \ref{Tab:EXP1corridor} (left), Tab. \ref{Tab:data}). Coordination with fixed priorities 
(1 (red) $>$ 2 (blue)) yields conflict-free plans (Fig. \ref{Tab:EXP1corridor} (centre)). However, agent 2 is forced to wait too long at its start 
location to be able to reach intermediate waypoint $\state_{f,1}^2$ in time and therefore, decides to move directly to its second goal. 
This could spawn high social cost due to missing one of the designated goals (Tab. \ref{Tab:data} ). By contrast, the auction method is flexible enough 
to reverse the ranking at the detected collision point causing agent 1 to wait instead of 2 (Fig. \ref{Tab:EXP1corridor} (right)). 
Thereby, agent 2 is able to reach both of its goal states in time. This success is reflected by low social cost (see Tab. \ref{Tab:data}). 
\begin{figure}[ht] 
\centering
	 \includegraphics[scale=.4, clip, trim = 3cm 8.7cm 3cm 9cm]{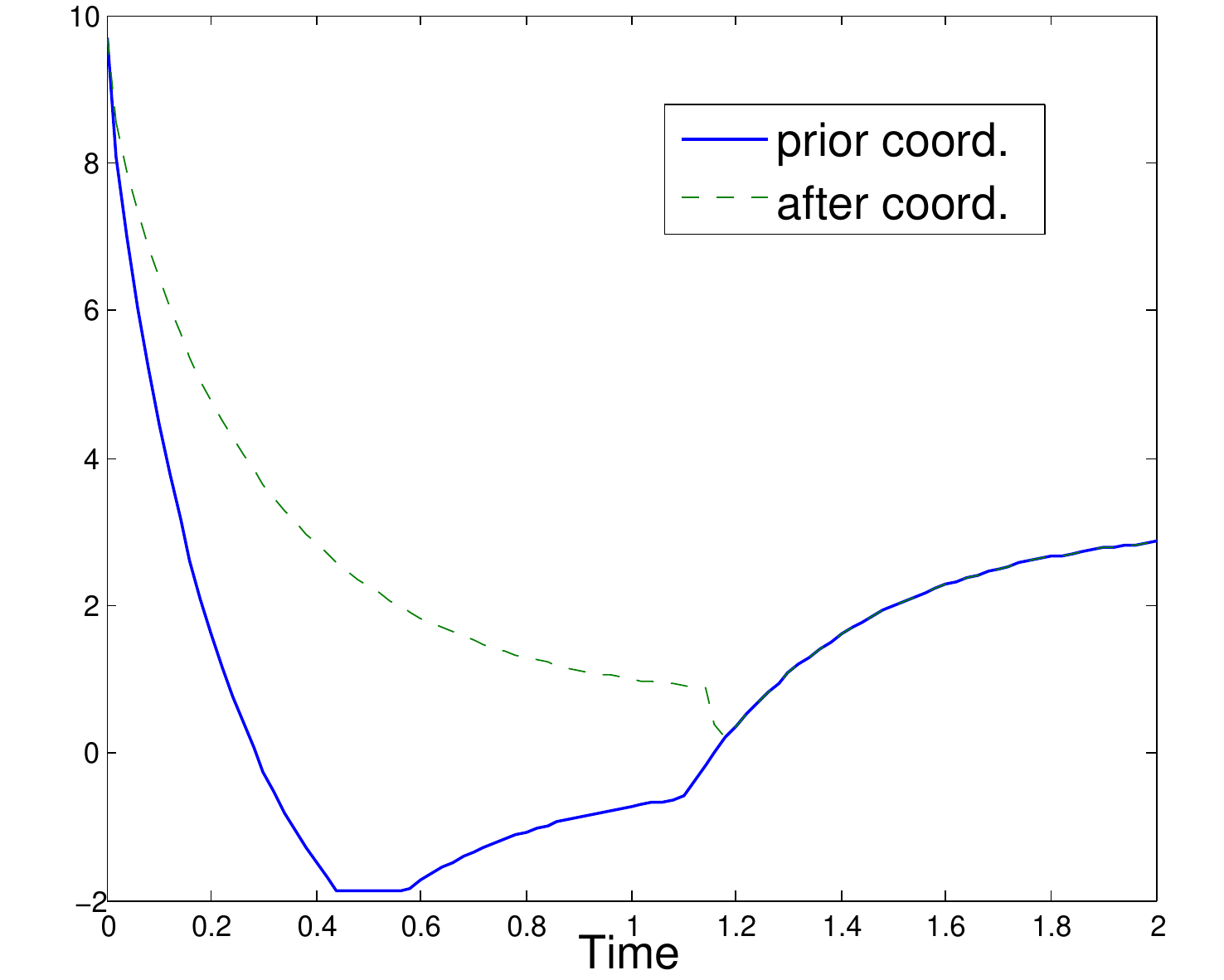}
	\caption{EXP1. Criterion functions for collision detection of agent 2 before and after coordination. The first graph accurately warns of a collision before coordination (as indicated by negative values), whereas the one corresponding to collision-free trajectories is in the positive-half space as desired.}
	\label{fig:CriterionFunctionsForCollisionDetectionOfAgent2BeforeAndAfterCoordination}
\end{figure}

\textbf{EXP2.} The setup was analogous to EXP1 but with three agents and different start and goal states as depicted in Fig. \ref{Tab:EXP2}. 
Furthermore, collisions were avoided with the FREE method with 10 random initializations of the local optimizer.  
Coordination of plans with fixed priorities (1 (red) $>$ 2 (blue) $>$ 3 (green) ) caused 2 to avoid agent 1 by moving to the left. 
Consequently, 3 now had to temporarily leave its start and goal state to get out of the way (see Fig. \ref{Tab:EXP2} (centre) ). 
With two agents moving to avoid collisions social cost was relatively high (see Tab. \ref{Tab:data}). During coordination with the auction-based 
method agent 2 first chose to avoid agent 1 (as in the FP method). However, losing the auction to agent 3 at a later stage of coordination, 
agent 2 decided to finally circumvent 1 by arcing to the right instead of to the left. This allowed 3 to stay in place (see Tab. \ref{Tab:data}).

%
%
%
%
%
%

\textbf{EXP3.} Next, we conducted a sequence of experiments for varying numbers of agents ranging from $|\agset|=1,..,7$. 
In each experiment all agents' start locations were placed on a circle. Their respective goals 
were placed on the opposite ends of the circle. The eigenvalues of the feedback gain matrices of each agent were drawn at random from a uniform distribution on the range [2,7].
An example situation for an experiment with 5 agents is depicted in Fig. \ref{Tab:EXP35agents}.
Collision avoidance was achieved.

Note, that despite this setting being close to worst case (i.e. almost all agents try to traverse a common, narrow corridor) the coordination overhead is moderate (see Fig. \ref{Tab:exp3barplots}, right). Also, all collisions were successfully avoided (see Fig. \ref{Tab:exp3barplots}, left).

\section{ Conclusions}
This work considered multi-agent planning under stochastic uncertainty and non-convex chance-constraints
for collision avoidance. 
In contrast to pre-existing work, we did not need to rely on prior space or time-discretisation. This was achieved by deriving criterion functions with the property that the collision probability is guaranteed to be below a freely definable threshold $\delta \in (0,1)$ if the criterion function attains no negative values. Thereby, stochastic collision detection is reduced to deciding whether such negative values exist. For Lipschitz criterion functions, we provided an algorithm for making this decision rapidly. 
We described a general procedure for deriving criterion functions and presented two such functions based on Chebyshev-type bounds.  
The advantage of using Chebyshev inequalities is their independence of the underlying distribution. Therefore, our approach is applicable to any stochastic state noise model for which the first two moments can be computed at arbitrary time steps. In particular, this would apply to models with state-dependent uncertainty and non-convex chance constraints which, to the best to our knowledge, have not been successfully approached in the multi-agent control literature. 
Nonetheless, future work could build on our results and derive less conservative criterion functions by using more problem-specific probabilistic inequalities. For instance, in simple cases such as additive Gaussian noise, tighter bounds can be given \cite{Blackmore2006} and used in Eq. \ref{eq:critfctgeneric}.

To enforce collision avoidance, our method modified the agent's plans until no collisions could be detected. To coordinate the detection and avoidance efforts of the agents, we employed an auction-based as well as a fixed-priority method.\\

Our experiments are a first indication that our approach can succeed in finding collision-free plans with high-certainty with the number of required coordination rounds scaling mildly in the number of agents. While in its present form, the coordination mechanism does not come with a termination guarantee, in none of our simulations have we encountered an infinite loop. For graph routing, \cite{ArmsTR:2011} provides a termination guarantee of the lazy auction approach under mild assumptions. Current work considers if their analysis can be extended to our continuous setting. Moreover, if required, our approach can be combined with a simple stopping criterion that terminates the coordination attempt when a computational budget is expended or an infinite loop is detected. 

The computation time within each coordination round depends heavily on the time required for finding a new setpoint and for collision detection. This involves minimizing $ (t,s) \mapsto c_{plan}^\agi(p^\agi_{\uparrow(t,s)})$ and $c_{coll}^\agi$, respectively.  The worst-case complexity depends on the choice of cost functions, their domains and the chosen optimizer. Fortunately, we can draw on a plethora of highly advanced global optimisation methods (eg \cite{Shubert:72,direct:93}) guaranteeing rapid optimization success. 
In terms of execution time, we can expect considerable alleviations from implementation in a compiled language. Furthermore, the collision detection and avoidance methods are based on global optimization and thus, would be highly amenable to parallel processing -- this could especially benefit the auction approach. 

While our exposition was focussed on the task of defining setpoints of feedback-controlled agents, the developed methods can be readily applied to other policy search settings, where the first two moments of the probabilistic beliefs over the trajectories (that would result from applying the found policies) can be computed.



\begin{small}
\bibliographystyle{plain}

\end{small}
\newpage
\appendix{Supplementary derivations.}
\section{Derivations of the covariance and mean of the feedback-controlled agents as 
deployed in the trajectory planning experiments }
We will now solve the mean and covariance for the dynamics given by the Ito-SDE describing the dynamics of the feedback controlled agents
considered in the experiments of this paper. Our aim is to obtain closed-form solutions avoiding the need to approximate any integrals.
This ameliorates the computational burden of our method since mean and covariance matrix functions need to be evaluated frequently in the course of 
collision detection and resolution.

\begin{thm}
For all $t \in [t_0,T]$ let $\xi(t),\state(t) \in \Real^D$, $A = \text{diag}(a_1,...,a_D), K = \text{diag}(k_1,...,k_D)$, $B = \text{diag}(\sqrt \nu_1,...,\sqrt \nu_D)$ and let $x(t_0)$ be a normally distributed random vector. The solution to the SDE
$\mathbf{d\state = (A \state - K (\state -\xi)) dt + B \, dW}$ is a Gaussian process with 
vector-valued mean function $\mu: [t_0,T] \to \Real^D$ and matrix-valued covariance function $C: [t_0,T]^2 \to \Real^{D\times D}$. Here the mean components are \[\mu_j(t) = e^{(k_j-a_j)(t_0-t)} \, \expect{\state_j(t_0) } + \int_{t_0}^t k_j\, e^{(k_j-a_j)(\tilde t-t)} \xi_j(\tilde t) d \tilde t \] and 
the covariance matrix function is $C(s,t) = \text{diag}(\text{cov}_{11}(s,t),...,\text{cov}_{DD}(s,t))$ where 
\begin{align*}
\text{cov}_{jj}(s,t) &= e^{(a_j-k_j) (t+s - 2 t_0)} (\expect{\state_j^2(t_0)} -  \expect{\state_j(t_0)}^2 )  \\
&+\frac{\nu_j}{2(k_j-a_j) }  [e^{(a_j-k_j) \abs{t-s}} - e^{(k_j-a_j)( 2 t_0 - (s+t))}].
\end{align*}

\end{thm}
\begin{proof}
Let $q_j := k_j-a_j$.
Owing to the diagonal form of $A,K$ and $B$ and the independence of the output dimensions of the vector-valued Wiener process, the given SDE decomposes into a system of $D$ indpendent 1-dimensional SDEs 
$d\state_j = (- q_j \state_j + k_j \xi_j) dt + \sqrt \nu_j dW_j$ ($j=1,...,D$) which can be treated separately.

For ease of notation we omit the subscripts yielding an SDE of the form:
$d\state = (- q \state + k \xi) dt + \sqrt \nu \, dW$.
To solve each of these SDEs we introduce the substitution $y := \state e^{q t}$. 
With Ito's product rule we find 
\begin{align*}
dy &= \state \, de^{q t} + e^{q t} \, d\state + d\state \, d e^{q t}\\
&=  \xi  k e^{q t} \,dt + \sqrt \nu e^{q t} dW 
\end{align*}

 where the last equality follows from substitution of the SDE for $dx$ and utilization the formal rules $dt \, dW =0, (dt)^2 =0$.

This SDE is solved by 
$y(t) = y(t_0) + \int_{t_0}^t k \xi(\tilde t) e^{q \tilde t} d\tilde t + \int_{t_0}^t  e^{q \tilde t} \sqrt \nu dW(\tilde t)$ where the last integral is to be interpreted as an Ito integral. Re-substitution yields 
 
\[\state(t) = \state(t_0) e^{q (t_0-t)} + \int_{t_0}^t k\,e^{q (\tilde t-t)} \xi(\tilde t) d \tilde t +  \int_{t_0}^t  e^{q (\tilde t - t)} \sqrt \nu dW(\tilde t). \]
Random variables $\state(t)$ $(t \in [t_0,T])$ are given as affine transformations of a Wiener process. Since the Wiener process is Gaussian, the solution is clearly a Gaussian process as long as $\state(t_0)$ is a normally distributed random variable.

For notational convenience we define 
\begin{align*}
J_t := \int_{t_0}^t k\,e^{q (\tilde t-t)} \xi(\tilde t) \,d \tilde t,\\ 
\Omega_t:= \int_{t_0}^t  e^{q (\tilde t - t)} \sqrt \nu\, dW(\tilde t)  \\
\text{and   }c_t := \state(t_0) e^{q (t_0-t)}.
\end{align*}
Therefore, we can write 

\[ \state(t) = c_t + J_t + \Omega_t. \]

It remains to find the first and second moments. Owing to the linearity of the expectation and, since the expectation of the Ito integral of a non-anticipating integrand is zero, the mean function is:
\[\mu(t) = \expect{\state(t)} = \expect{c_t} + \expect{J_t}+ \expect{\Omega_t} = \expect{\state(t_0)} e^{k (t_0-t)} + J_t.\]

Now, we calculate the second moment: 
\begin{align*}
\cov s t &= \expect{\state(s) \state(t) }  -\mu(s) \mu(t)\\
\end{align*}
where 

\begin{align*}
\expect{\state(s) \state(t) }&= \expect{(c_s+J_s+\Omega_s)(c_t+J_t+\Omega_t)  } \\
&= \expect{c_s c_t}+ \expect{c_s J_t} + \expect{c_s \Omega_t}+...+\expect{\Omega_s \Omega_t}.
\end{align*}
Once again leveraging that the expectation of an Ito integral with non-anticipating integrand is zero, the cross-terms $\expect{c_q \Omega_r},\expect{J_q \Omega_r}$ ($q,r \in \{s,t\}$) vanish. Hence, we 
obtain
\begin{align*}
\expect{\state(s) \state(t) } 
&= \expect{c_s c_t}+ \expect{c_s J_t} + \expect{c_t J_s}+J_t J_s+\expect{\Omega_s \Omega_t}\\
&=\expect{c_s c_t} + J_t \expect{c_s} + J_s \expect{c_t}+J_t J_s+\expect{\Omega_s \Omega_t}\\
&=\mu(s) \mu(t) + \expect{c_s c_t} - \expect{c_s} \expect{c_t} +\expect{\Omega_s \Omega_t}.
\end{align*}

Hence, 
\begin{align*}
\cov s t &= \expect{c_s c_t} - \expect{c_s} \expect{c_t} +\expect{\Omega_s \Omega_t}\\
 &= e^{-q (t+s - 2 t_0)} (\expect{\state^2(t_0)} -  \expect{\state(t_0)}^2 )  +\expect{\Omega_s \Omega_t}
\end{align*}

It is well-known that for non-anticipating $f,g$ interval $I$ we have \[\expect {\int_{I} f(t) dW(t) \int_{I} g(t) dW(t) } = \int_I f(t) g(t) dt.\] Applying this fact as well as leveraging the independent increments property of the Wiener process yields:
\begin{align*}
\expect{\Omega_s \Omega_t} &= \frac{\nu}{e^{q(s+t)}}  \int_{t_0}^{\min\{t,s\}} e^{2q\tilde t} d \tilde t \\
&=  \frac{\nu}{2q e^{q(s+t)}}  [e^{2k\min\{t,s\}} - e^{2q t_0}] \\
&=  \frac{\nu}{2q }  [e^{q (2\min\{t,s\} - (s+t)} - e^{q( 2 t_0 - (s+t))}] \\
&=  \frac{\nu}{2q }  [e^{-q \abs{t-s}} - e^{q( 2 t_0 - (s+t))}]
\end{align*}

\end{proof}

Notice, that when altering the plan and hence, $\xi$, the integral 
\[J_{t,j} := \int_{t_0}^t k_j\,e^{(k_j-a_j) (\tilde t-t)} \xi_j(\tilde t) d \tilde t\]
has to be computed. For general forms of allowable $\xi$ we would have to rely on numerical approximation methods. Since repeated calculations of solutions need to be done in the course of coordination we will want to alleviate the computational burden thereof as much as possible. This motivated our restriction to plans that give rise to step functions for which we can show that $J_t$ is closed-form.  
\begin{lem} Let $t \geq t_0, t, t_0 \in [0,T^\agi]$.
Given plan $p^\agi = \bigl((t_i^\agi,\zeta_i^\agi)\bigr)_{i=0}^{H^\agi}$ where each $\zeta^\agi_{i} = (\zeta^\agi_{i,j})_{j=1}^D$. Let $\underline i = \arg\max_i \{t_i \leq t_0 \}$,  $\overline i = \arg\min_i \{t_i \geq t \}$ and $\mathcal I := \{ i \in \{1,...,H^\agi\} |  \underline i < i \leq \overline i  \}$. Furthermore, let $\xi^\agi_j$ denote the jth component of step-function setpoint signal $\xi^\agi$.
We have 
\begin{align*}
&\int_{t_0}^t k_j\,e^{q_j(\tilde t-t)} \xi_j^\agi(\tilde t)  d \tilde t \\
&= \sum_{i \in \mathcal I} \frac{k_j}{q_j} \zeta_{i,j}^\agi (e^{q_j (\min\{t_i,t\} -t) } - e^{q_j (\max\{t_{i-1},t_0\} -t)}).
\end{align*}
\end{lem}
\begin{proof}
Remember,
$\tau_{i}^\agi = (t_{i-1}^\agi, t_{i}^\agi] \subset [0,T^\agi]$.
We have 
\begin{align*}
&\int_{t_0}^t k_j\,\exp(q_j(\tilde t-t)) \xi_j^\agi(\tilde t)  d \tilde t \\
& =\int_{t_0}^t k_j\,\exp(q_j(\tilde t-t)) \sum_{i=1}^{H^\agi} \zeta_{i,j}^\agi \chi_{\tau_i^\agi} (\tilde t) d \tilde t \\
& = \sum_{i \in \mathcal I} \int_{\max\{t_0,t_{i-1}\}}^{\min\{t,t_i \}} k_j\,\exp(q_j(\tilde t-t))  \zeta_{i,j}^\agi d \tilde t \\
\end{align*}
Calculation of the anti-derivate and substitution of the integration bounds yields the desired result. 

\end{proof}

We immediately get the following corollary
\begin{cor}
\label{thm:corlinSDEFBindepoutputcomp}
Let $\agi \in \agset$ be an agent with controlled plant dynamics \[d\state^\agi = (A \state^\agi - K (\state^\agi -\xi^\agi)) dt + B \, dW \] where for all $t \in [t_0,T]$: $\xi^\agi(t),\state^\agi(t) \in \Real^D$, $A = \text{diag}(a_1,...,a_D)$, 

$K = \text{diag}(k_1,...,k_D)$, $B = \text{diag}(\sqrt \nu_1,...,\sqrt \nu_D)$. Let $\state^\agi(t_0)$ be a normally distributed random vector. Assume $\agi$'s plan is $p^\agi = \bigl((t_i^\agi,\zeta_i^\agi)\bigr)_{0=1}^{H^\agi}$ where each $\zeta^\agi_{i} = (\zeta^\agi_{i,j})_{j=1}^D$. Let $t \geq t_0$, $\underline i = \arg\max_i \{t_i \leq t_0 \}$,  $\overline i = \arg\min_i \{t_i \geq t \}$ and $\mathcal I := \{ i \in \{1,...,H^\agi\} |  \underline i < i \leq \overline i  \}$. Furthermore, let $\xi^\agi_j$ denote the jth component of step-function reference signal $\xi^\agi$.

The solution to agent $\agi$'s SDE
is a Gaussian process with vector-valued mean function $\mu: [t_0,T] \to \Real^D$ and matrix-valued covariance function $C: [t_0,T]^2 \to \Real^{D\times D}$. 
Here the mean components are 

\begin{align*}
&\mu_j(t) = e^{k_j(t_0-t)} \, \expect{\state_j(t_0) } \\
&+ \sum_{i \in \mathcal I} \frac{k_j}{k_j-a_j} \zeta_{i,j}^\agi (e^{(k_j-a_j) (\min\{t_i,t\} -t) } - e^{(k_j-a_j) (\max\{t_{i-1},t_0\} -t) }) 
\end{align*} 
and 
the covariance matrix function is $C(s,t) = \text{diag}(\text{cov}_{11}(s,t),...,\text{cov}_{DD}(s,t))$ where 
\begin{align*}
\text{cov}_{jj}(s,t) &= e^{(a_j-k_j) (t+s - 2 t_0)} (\expect{\state_j^2(t_0)} -  \expect{\state_j(t_0)}^2 )  \\
&+\frac{\nu_j}{2(k_j-a_j) }  [e^{(a_j-k_j) \abs{t-s}} - e^{(k_j-a_j)( 2 t_0 - (s+t))}].
\end{align*}
\end{cor}

\section{Deriving Lipschitz numbers}
\label{sec:derlipno}
So far, we have established 
how knowledge of the Lipschitz number of a Lipschitz function can be utilized to exclude the presence of 
any negative function values on a compact domain based on a finite number of function evaluations.
To employ this insight in the context of collision detection we will have to know a Lipschitz number $L$ of the criterion functions $\gamma^{\agi,\agii}$.

We may consider two cases: 

\begin{itemize}
	\item We have a belief quantified as a distribution over the smallest Lipschitz constant $L'$. Let the cumulative distribution function of this belief be denoted by $F : \Real \to [0,1]$. That is, $F(x) = \Pr[L' < x ]$. If we desire a guaranteed success of collision detection of at least $\delta \in (0,1)$ we invoke $Alg. \ref{alg:negdetect_lipschitz}$ with a Lipschitz number $L \geq \min\{ x \in \Real | F(x) \geq \delta\}$ to detect non-positive values of $\gamma^{\agi,\agii}$.

\item If we do not have such a belief function or desire complete certainty in collision detection success it may be possible to derive a Lipschitz number for the $\gamma^{\agi,\agii}$ based on the mean and covariance functions of the agent's stochastic trajectories. In turn, these may be derived from the agents' SDEs. How this can be accomplished is the subject of the remainder of this section.
\end{itemize}

\subsection{Lipschitz arithmetic} 
In preparation of the derivations of Lipschitz numbers, we need to establish a few basic properties of Lipschitz continuous functions. As a convention, $L_\phi$ will always denote the Lipschitz constant of a 
Lipschitz continuous function $\phi$.
\begin{lem}[Lipschitz arithmetic] \label{lem:Liparithmetic}
Let, $I,J \subset \Real_+$. Let $f : \Real \to \Real$ be Lipschitz on $I$ with Lipschitz number $L_I (f)$ 
and $g :\Real \to \Real$ be Lipschitz on $J$ with Lipschitz number $L_J (g)$.
We have:

\begin{enumerate}
	\item Mapping $t \mapsto |f(t)|$ is Lipschitz on $I$ with constant $L_I(f)$.
	\item If $g$ is Lipschitz on all of $J=f(I)$ the concatenation $g \circ f: t \mapsto g(f(t))$ is Lipschitz on $I$ with Lipschitz constant 
	      $L_I(g \circ f) \leq$ $L_J (g) \, L_I(f)$.
	\item Let $r \in \Real$. $r \, f: x \mapsto r \, f(x)$ is Lipschitz on $I$ having a Lipschitz constant $L_I (r \,f) = |r| \, L_I(f)$.
	\item $f+g: t \mapsto f(t) + g(t)$ has Lipschitz number at most $L_I(f) + L_J(g)$.
	\item Let $m_f = \sup_{t\in I} f(t)$ and $m_g = \sup_{t\in I} g(t)$. Product function $f\cdot g: x \mapsto f(x) \, g(x)$ has a Lipschitz number on $I$ which is at most $(m_f \, L_J(g)+ m_g \, L_I(f))$.
	\item Let $h(t) = \max\{f(t), g(t) \}, \forall t \in I \cap J$. We have $L_{I \cap J}(h) = \max\{L_I(f),L_J(g)\}$.
	\item Let $b := \inf_{t \in I}| f(t)| > 0$ and let 
	$\phi(t) = \frac{1}{f(t)}, \forall t \in I$.  
	      Then $L_I(\phi) \leq b^{-2} \, L_I(f)$ on $I$.  
	\item $f$ cont. differentiable on I $\Rightarrow$ $L_I(f) = \sup_{t \in I} |\dot f(t)|. $ 
	 \item Let $c \in \Real$, $f( t) = c, \forall t \in I$. Then $L_I(f) =0$.  
	\item $L_I(f^2) \leq 2 \, s(f)\, L_I(f)$.
	\item $f$ cont. differentiable $\Rightarrow \forall q \in \mathbb Q : L_I(f^q) = |q| \,\sup_{\tau \in I} |f^{q-1}(\tau) \, \dot f(\tau)| $.
\end{enumerate}
\begin{proof}

\textbf{1)}  We show $|f|$ has the same Lipschitz number as $f$. Let $t,t' \in I$ arbitrary. 
We enter a case differentiation:

\textit{1st case: $f(t), f(t') \geq 0$}. 

Hence, $\bigl| |f(t)| - |f(t')| \bigr| = \bigl| f(t) - f(t') \bigr|  \stackrel{f \, Lipschitz}{\leq} L_I(f) |t-t'|$.\\

\textit{2nd case: $f(t) \geq 0, f(t') \leq 0$.} 

Note, $|y| = - y$, iff $y \leq 0$. Hence,  $\bigl| |f(t)| - |f(t')| \bigr| \leq \bigl| |f(t)| + |f(t')| \bigr| $
$= \bigl| |f(t)| - f(t') \bigr|  =  \bigl| f(t) - f(t') \bigr| \leq L_I(f) \, |t-t'|$.\\

\textit{3rd case: $f(t) \leq 0, f(t') \geq 0$.} Completely analogous to 2nd case.\\

\textit{4th case: $f(t), f(t') \leq 0$}. 

$\bigl| |f(t)| - |f(t')| \bigr| = \bigl| f(t') - f(t) \bigr|= \bigl| f(t) - f(t') \bigr|  \stackrel{f \, Lipschitz}{\leq} L_I(f) |t-t'|$.\\

\textbf{2)}  For arbitrary $t,t' \in I$ we have:

$\bigl| g(f(t)) - g(f(t'))| \bigr| \leq L_J(g) \, |f(t) - f(t') | \leq L_J(g) \, L_I(f)\, |t-t'|$ 
where the two inequalities are due to the Lipschitz properties of $g$ and $f$, respectively.\\

\textbf{3)}  For arbitrary $t,t' \in I, r \in \Real$ we have:

$\bigl| r \, f(t) - r \, f(t')| \bigr| = |r|\, |f(t) - f(t')| \leq |r| \,L_I(f)\,  |t-t'|$.\\ 

\textbf{4)}  For arbitrary $t,t' \in I, r \in \Real$ we have:

$\bigl| g(t) + f(t) - (g(t') + f(t'))| \bigr| = \bigl| g(t)  - g(t') + f(t)- f(t')| \bigr|$ 
$\leq \bigl| g(t)  - g(t')\bigr|  + \bigl| f(t)- f(t')| \bigr|$ $\leq (L_J(g)+L_I(f))\,  |t-t'|$.\\

\textbf{5)}  Let  $t,t' \in I$, $d := f(t) - f(t')$.

$\bigl| g(t) f(t) - g(t')  f(t') \bigr| = \bigl| g(t) (f(t') +d) - g(t') f(t') \bigr|$ 
$= \bigl|\bigl( g(t) - g(t') \bigl)  f(t')+ g(t)  d \bigr|  $
$\leq \bigl| g(t) - g(t') \bigr|  |f(t')|   + \bigl|g(t)\bigr| \,  |d|  $
$\leq L_I(g) |t-t'|  |f(t')|   + \bigl|g(t)\bigr| \,  L_I(f) |t-t'|  $
$\leq L_I(g) |t-t'|  \sup_{t' \in I} \{|f(t')|\}   + \sup_{t \in I}\{\bigl|g(t)\bigr|\} \,  L_I(f) |t-t'|  $
$= \Bigl(L_I(g)  \sup_{t' \in I} \{|f(t')|\}   + \sup_{t \in I}\{\bigl|g(t)\bigr|\} \,  L_I(f)\Bigl) |t-t'|  $.\\

\textbf{6)}  Proof in "` Nick Weaver, Lipschitz algebras"'.

\textbf{7)}  Let  $t,t' \in I$.
$\bigl| \frac{1}{f(t)} - \frac{1}{f(t')} \bigr|$ 
$=\bigl| \frac{f(t')}{f(t') f(t)} -\frac{f(t)}{f(t') f(t)} \bigr|$ 
$= \frac{\bigl|f(t')-f(t) \bigr|}{\bigl|f(t')\bigr| \bigr| f(t)\bigr|}$ 
$\leq \frac{L_I(f) |t-t'|}{\inf_{t \in I} |f(t)|}$.\\

\textbf{8)} Define $\ell := \sup_{t \in I} | \dot f(t) | = L_I(f)$. In two steps, we show that $\ell$ is the smallest Lipschitz constant.

Firstly, we show that it is a Lipschitz constant: Let $t,t' \in I, t < t'$. Due to the mean value theorem $\exists \xi \in [t,t'] \subset I: \frac{|f(t) - f(t')|}{|t-t'|} = | \dot f (\xi) | \leq \ell$. 
Secondly, we show that $\ell$ is the smallest Lipschitz constant: Let $\bar \ell$ be another Lipschitz constant such that $\bar \ell \leq \ell$. Since $I$ is compact and $\dot f$ is continuous there is some $\xi \in I$ such that $\dot f(\xi) = \ell$. Pick any sequence $(t_k)_{k=1}^\infty$ such that $t_k \stackrel{k \to \infty}{\longrightarrow} \xi$.
$\forall k: t_k \in I$ and $\bar \ell$ is a Lipschitz constant on $I$. Hence, $ \bar \ell \geq \frac{|f(t_k) - f(\xi)|}{|t_k-\xi|}\stackrel{k \to \infty}{\longrightarrow} | \dot f(\xi)| = \ell$. Thus, $\bar \ell = \ell$.\\

\textbf{9)} Immediate consequence of 8).\\

\textbf{10)} $L(f^2) = L(f \, f) \leq s(f) L(f) + L(f) s(f)$ where the last inequality follows from property 5).

\textbf{11)} $L(f^q) \stackrel{8)}{=} \sup_{\tau \in I} |\dif{}{t} f^q(\tau)| = |q| \,\sup_{\tau \in I} |f^{q-1}(\tau) \, \dot f(\tau)| $. 

\end{proof}
\end{lem} 

Notice, that several of the inequalities in the Lemma are not tight (Eg. Ineq. (5), (10) ). Therefore, it may sometimes be better not to apply it if instead one is able to determine the Lipschitz constant directly to yield a lower Lipschitz number.

%
%

\begin{lem}\label{lem:lipschitzsqrt}
 For $0 < a < b$ let $J \subset \Real_+$ be the domain interval of square root function $\sqrt \cdot : J \to \Real_+, t \mapsto \sqrt{t}$ 
such that $\inf J =a$ 
and $\sup J =b$. We have $L_J(\sqrt \cdot) \leq \frac{1}{2 \sqrt {a}}$ where 
  $L_J(\sqrt \cdot)$ denotes the Lipschitz constant of the square root function on $J$.
\begin{proof}
 Applying Lem. \ref{lem:Liparithmetic} and leveraging differentiability of the square root function reduces the problem of determining a Lipschitz constant to finding 
 $\sup_{t \in J}  |\frac{d \sqrt{\cdot}}{d t}(t)|= \sup_{t \in J} |\frac{1}{2 \sqrt{t}}|$. Since $\frac{d^2 }{d t^2} \sqrt{\cdot} = -\frac{1}{4 \sqrt \cdot ^3} $ 
attains only negative values on $J \subset \Real_+$ we know that the first derivative $\frac{d \sqrt{\cdot}}{d t}=\frac{1}{2 \sqrt{\cdot}}$ 
is strictly monotonously decreasing. Thus, $\sup_{t \in J} |\frac{d \sqrt{\cdot}}{d t}(t)| = 
\frac{d \sqrt{\cdot}}{d t} (\inf J)= \frac{1}{2 \sqrt {\inf J}} = \frac{1}{2 \sqrt {a}}$. 
\end{proof}

\end{lem}


\subsection{An alternative collision-criterion function and the derivation of its Lipschitz number}
\label{sec:Lipno_collcritfct2d_gen}
\begin{thm}[Alternative collision criterion function] \label{def:collcritfct2d}
 Let spatial dimensionality $D = 2$. Let $\delta^\agi \in (0,1)$ denote the maximum upper bound on instantaneous collision probability at time $t$ agent $\agi \in \agset$ is allowed to tolerate. 
Let $\mu^\agi(t) \in \Real^D$ be the mean of trajectory $\state^\agi$ and $C_{ij}^\agi(t)$ be the spatial between dimensions $i$ and $j$ at time $t$.
For $i \in \{1,2\}, j \in \{1,2\} - \{i\}$ let
\noindent
$r^{\agi}_i(t) := \sqrt{\frac{1}{2\delta^a}}\, \sqrt{ C_{ii}^\agi(t) +\frac{ \sqrt{C_{ii}^\agi(t) C_{jj}^\agi(t) 
(C_{ii}^\agi(t) C_{jj}^\agi(t) - (C_{ij}^\agi(t))^2) }}{C_{jj}^\agi(t)}}$.

Let $\agi,\agii$ be two agents' plants whose radii sum to $\Lambda^{\agi,\agii}$ with state trajectories $\state^\agi,\state^\agii$, respectively.  Define \[b^{\agi,\agii}_j(t,t') :=r^\agi_j(t) + r^\agii_j(t') + \Lambda^{\agi,\agii}. \]
The function
\[ \gamma^{\agi,\agii}(t) := \max_{j \in \{1,...,D\}}\{|\mu^\agi_j(t) -\mu^\agii_j(t) | - b^{\agi,\agii}_j(t,t)   \}\] is a valid criterion function. That is, $\gamma^{\agi,\agii}(t) >0 \Rightarrow \Pr [\mathfrak C^{\agi,\agii}(t)]  < \delta^\agi$.
\end{thm}
\begin{proof}(Sketch)  Let $t \in I$ be an arbitrary but fixed time.
It is straight-forward to adapt the proof of Thm. 2 in \cite{Lyons2011} to our case showing that 
$\Pr [\mathfrak C^{\agi,\agii}(t)]  < \delta^\agi$ if $|\mu^\agi_j(t) -\mu^\agii_j(t,t') | - b_j(t,t') > 0$ for at least one $j \in \{1,\ldots,D\}$. Hence, $\{t \in I | \gamma^{\agi,\agii}(t,t) > 0\} \subset \{t \in I | \Pr[\mathfrak C^{\agi,\agii}(t)] < \delta^\agi \}$.
\end{proof}
\begin{defn}
For future reference we define $s: f \mapsto \sup_{t \in I} |f(t)| $ and $\iota : f \mapsto \inf_{t \in I} |f(t)| $ on the space of continuous functions on interval $I$.
\end{defn}

\begin{thm} \label{thm:Lipschitzno_whittle_2d}
Let $D=2$ be the dimensionality of state space. 
For any agent $\agi$ let $\mu_j^\agi: I=[t_0, t_f] \to \Real$ denote the $j$th component of $\agi$'s mean function and $C^\agi_{ij}(t)$ the covariance of agent $\agi$'s trajectory between dimension $j$ and $i$ at time $t$.

For $\mathfrak q \in \{\agi,\agii\}$ and all  $i,j \in \{1,\ldots,D\}$ assume the $\mu_j^\agiii, C_{ij}^\agiii$ are Lipschitz on $I$ with Lipschitz numbers $L(\mu_j^\agiii)$ and  Lipschitz numbers $L(C_{ij}^\agiii)$, respectively.

Let $\gamma^{\agi,\agii}$ denote the collision criterion function between agents $\agi$ and $\agii$ as defined in Thm. \ref{def:collcritfct2d}.

Then $\gamma^{\agi,\agii}$ is Lipschitz on $I$ with Lipschitz constant 
 
\begin{align*} L(\gamma^{\agi,\agii}) 
&\leq \max_{j \in \{1,...,D\} } \{L(\mu_j^\agi -\mu_j^\agii) + L(b_j^{\agi,\agii}) \}\\
&\leq \max_{j \in \{1,...,D\} } \{L(\mu_j^\agi) + L(\mu_j^\agii) + L(b_j^{\agi,\agii}) \}  
\end{align*}
 
where $\gamma^{\agi,\agii}, b^{\agi,\agii}, \alpha^{\agi,\agii}$ are defined as in Def. \ref{def:collcritfct2d} and

$\forall i \in \{1,2\}, j\in \{1,2\} -\{i\}$ $	\forall \mathfrak q \in \{\agi,\agii\}$ we have:
\begin{enumerate}
	 \item $L(b_j^{\agi,\agii}) \leq \frac{1}{2} \bigl( L(\alpha_j^\agi )+ L(\alpha_j^\agii ) \bigr)$,	
\item 

$L(\alpha_i^{{\mathfrak q}}) \leq \sqrt{\frac{1}{2 \delta^\agiii}}\,\frac{1}{ \iota (g_i) } L(  g_i ) $

where $g_i(t) :=  C_{ii}^{\mathfrak q}(t) +\sqrt{ (C_{ii}^{\mathfrak q})^2  - C_{ii}^{\mathfrak q} \frac{\bigl( C_{ij}^{\mathfrak q} \bigr)^2}{C_{jj}^{\mathfrak q}} }$ and 
 
where $(i)\, L(g_i) \leq L(C_{ii}^{\mathfrak q}) +  Q  \, L\Bigl((C_{ii}^{\mathfrak q})^2\Bigr)  + Q \, s( \bigl( C_{ij}^{\mathfrak q} \bigr)^2 ) L\Bigl(  \frac{C_{ii}^{\mathfrak q}}{C_{jj}^{\mathfrak q}} \Bigr) + Q \, L( \bigl( C_{ij}^{\mathfrak q} \bigr)^2 ) s\Bigl(  \frac{C_{ii}^{\mathfrak q}}{C_{jj}^{\mathfrak q}} \Bigr)$, 
and 

$ (ii) \, L(g_i) \leq L(C_{ii}^{\mathfrak q}) +  Q  \, L\Bigl((C_{ii}^{\mathfrak q})^2\Bigr)  + Q \, s( C_{ii}^{\mathfrak q} ) L\Bigl(  \frac{\bigl(C_{ij}^{\mathfrak q}\bigr)^2}{C_{jj}^{\mathfrak q}} \Bigr) + Q \, L\bigl( C_{ii}^{\mathfrak q} \bigr)  s\Bigl(  \frac{\bigl(C_{ij}^{\mathfrak q}\bigr)^2}{C_{jj}^{\mathfrak q}} \Bigr)$. Here, $Q = \inf_{t \in I} \Bigl( (C_{ii}^{\mathfrak q}(t) )^2  - C_{ii}^{\mathfrak q}(t) \frac{\bigl( C_{ij}^{\mathfrak q}(t) \bigr)^2}{C_{jj}^{\mathfrak q}(t)} \Bigr) = \inf_{t \in I} (g_i-C_{ii}^{\mathfrak q})^2(t)$.

%
%
%
\item also, $L(\alpha_j^{{\mathfrak q}}) \leq \frac{1}{2\sqrt a} L \Bigl( \frac{C_{jj}^{\mathfrak q}}{C_{ii}^{\mathfrak q}} \Bigr) s(\alpha_i^{\mathfrak q}) 
+  s\Bigl( \sqrt{\frac{C_{jj}^{\mathfrak q}}{C_{ii}^{\mathfrak q}} }\Bigr) \, L(\alpha_1^{\mathfrak q})$, 
where $a = \inf_{t \in I} \frac{C_{jj}^{\mathfrak q}(t)}{C_{jj}^{\mathfrak q}(t)} $,

\item $L \Bigl( \frac{C_{ii}^{\mathfrak q}}{C_{jj}^{\mathfrak q}} \Bigr) \leq 	L(C_{jj}^{\mathfrak q}) \iota(C_{jj}^{\mathfrak q})^{-2} s(C_{ii}^{\mathfrak q}) + s(\frac{1}{C_{jj}^{\mathfrak q}}) \, L(C_{ii}^{\mathfrak q})$
if  $ \iota(C_{jj}) = \inf_{t \in I} |C_{jj}(t)| > 0$,

\item and similarly, $L \Bigl( \frac{(C_{ij}^{\mathfrak q})^2}{C_{jj}^{\mathfrak q}} \Bigr) \leq 	L(C_{jj}^{\mathfrak q}) \iota(C_{jj}^{\mathfrak q})^{-2} s((C_{ij}^{\mathfrak q})^2) + s(\frac{1}{C_{jj}^{\mathfrak q}}) \, L((C_{ij}^{\mathfrak q})^2)$
if  $ \iota(C_{jj}^\agiii) = \inf_{t \in I} |C_{jj}^\agiii(t)| > 0$.

\item $\forall i,j \in \{1,2\}:$ $L\bigl((C^{\mathfrak q}_{ij})^2\bigr) \leq 2 \, s(C^{\mathfrak q}_{ij}) \, 
L\bigl(C^{\mathfrak q}_{ij}\bigr)$.

	\end{enumerate}

\end{thm}
\begin{proof}

The equalities follow from successively applying Lem. \ref{lem:Liparithmetic} to the definitions of the parts of the criterion function given in 
Thm. \ref{def:collcritfct2d}. In our derivations we will note which of the properties of the Lemma we utilized as a superscript above the 
inequality sign. 
We have 

$L(\gamma^{\agi,\agii})$ \\
$\stackrel{def.}{=} L(  \max_{j \in \{1,...,D\}}\{|\mu^\agi_j -\mu^\agii_j | - b^{\agi,\agii}_j   \} )$\\
$\stackrel{(6)}{=} \max_{j \in \{1,...,D\}} L(  |\mu^\agi_j -\mu^\agii_j | - b^{\agi,\agii}_j   )$\\
$\leq \max_{j \in \{1,...,D\}} L(  |\mu^\agi_j -\mu^\agii_j |) + L( b^{\agi,\agii}_j   )$\\
$\stackrel{(1,3,4)}{\leq} \max_{j \in \{1,...,D\}} L( \mu^\agi_j ) + L(\mu^\agii_j) + L( b^{\agi,\agii}_j).   $\\

\textit{Proof of Ineq. 1):}
By definition, 
$L(b^{\agi,\agii}_j) = L(\frac{1}{2} (\alpha^\agi_j(t) + \alpha^\agii_j(t')) + \Delta^{\agi,\agii}) $\\
$\stackrel{(3,4)}{\leq} \frac{1}{2} (L( (\alpha^\agi_j) + L(\alpha^\agii_j)) + L(\Delta^{\agi,\agii}) \stackrel{(9)}{=} \frac{1}{2} (L( (\alpha^\agi_j) + L(\alpha^\agii_j))$.  

Furthermore, to prove the remaining inequalities, assume $\mathfrak q \in \{\agi, \agii\}$. \\

\textit{Proof of Ineq. 2): }
Let $g(t) :=  C_{ii}^{\mathfrak q}(t) +\sqrt{ (C_{ii}^{\mathfrak q})^2  - C_{ii}^{\mathfrak q} \frac{\bigl( C_{ij}^{\mathfrak q} \bigr)^2}{C_{jj}^{\mathfrak q}} }$.
Then,
$L(\alpha^{{\mathfrak q}}_1) \stackrel{(3)}{=} \sqrt{\frac{2}{\delta^a}}\, L( \sqrt g )$
$\leq  \sqrt{\frac{2}{\delta^\agiii}}\,\frac{1}{2 \iota (g) } L( g )$ 

where the last inequality follows from Lem. \ref{lem:lipschitzsqrt} and 
Lem.  \ref{lem:Liparithmetic}.(2) as before.
Furthermore, 
$L(g) \stackrel{(4)}{\leq}  L(C_{ii}^{\mathfrak q}) + L(\sqrt{ (C_{ii}^{\mathfrak q})^2  - C_{ii}^{\mathfrak q} \frac{\bigl( C_{ij}^{\mathfrak q} \bigr)^2}{C_{jj}^{\mathfrak q}} })$

$=  L(C_{ii}^{\mathfrak q}) + L\Bigl(\sqrt{ (C_{ii}^{\mathfrak q})^2  - C_{ii}^{\mathfrak q} \frac{\bigl( C_{ij}^{\mathfrak q} \bigr)^2}{C_{jj}^{\mathfrak q}} }\Bigr)$

$=  L(C_{ii}^{\mathfrak q}) +  Q  \, L\Bigl((C_{ii}^{\mathfrak q})^2  - C_{ii}^{\mathfrak q} \frac{\bigl( C_{ij}^{\mathfrak q} \bigr)^2}{C_{jj}^{\mathfrak q}} \Bigr)$
where according to 
Lem. \ref{lem:lipschitzsqrt},  $Q = \iota\Bigl( (C_{ii}^{\mathfrak q})^2  - C_{ii}^{\mathfrak q} \frac{\bigl( C_{ij}^{\mathfrak q} \bigr)^2}{C_{jj}^{\mathfrak q}} \Bigr)$ .
Hence, 
$L(g) \leq L(C_{ii}^{\mathfrak q}) +  Q  \, L\Bigl((C_{ii}^{\mathfrak q})^2\Bigr)  + Q \, L\Bigl( C_{ii}^{\mathfrak q} \frac{\bigl( C_{ij}^{\mathfrak q} \bigr)^2}{C_{jj}^{\mathfrak q}} \Bigr)$.\\

\textit{
Proof of Ineq. 3):}
We have  
$L(\alpha^{\mathfrak q}_j(t))\stackrel{def}{ =} L(\sqrt{\frac{C_{jj}^{\mathfrak q}}{C_{ii}^{\mathfrak q}}}\alpha_i^{\mathfrak q} )$ \\
$\stackrel {(5)}{\leq} s(\sqrt{\frac{C_{jj}^{\mathfrak q}}{C_{ii}^{\mathfrak q}}}) L(\alpha_i^{\mathfrak q})+ L(\sqrt{\frac{C_{jj}^{\mathfrak q}}{C_{ii}^{\mathfrak q}}}) s(\alpha_i^{\mathfrak q})$\\
$\stackrel{(2)}{\leq} s(\sqrt{\frac{C_{22}^{\mathfrak q}}{C_{ii}^{\mathfrak q}}}) L(\alpha_i^{\mathfrak q})+ L_{J}(\sqrt \cdot) L({\frac{C_{jj}^{\mathfrak q}}{C_{ii}^{\mathfrak q}}}) s(\alpha_i^{\mathfrak q}) $
where $J= \frac{C_{jj}^{\mathfrak q}}{C_{ii}^{\mathfrak q}}(I)$. 
Inequality 2) now follows from applying Lem. \ref{lem:lipschitzsqrt}. \\

From here we can make two alternative derivations:

i)
$L(g) \leq L(C_{11}^{\mathfrak q}) +  Q  \, L\Bigl((C_{11}^{\mathfrak q})^2\Bigr)  + Q \, s( \bigl( C_{12}^{\mathfrak q} \bigr)^2 ) L\Bigl(  \frac{C_{11}^{\mathfrak q}}{C_{22}^{\mathfrak q}} \Bigr) + Q \, L( \bigl( C_{12}^{\mathfrak q} \bigr)^2 ) s\Bigl(  \frac{C_{11}^{\mathfrak q}}{C_{22}^{\mathfrak q}} \Bigr)$.

Alternatively, one can obtain: 

ii)
$L(g) \leq L(C_{11}^{\mathfrak q}) +  Q  \, L\Bigl((C_{11}^{\mathfrak q})^2\Bigr)  + Q \, s( C_{11}^{\mathfrak q} ) L\Bigl(  \frac{\bigl(C_{12}^{\mathfrak q}\bigr)^2}{C_{22}^{\mathfrak q}} \Bigr) + Q \, L\bigl( C_{11}^{\mathfrak q} \bigr)  s\Bigl(  \frac{\bigl(C_{12}^{\mathfrak q}\bigr)^2}{C_{22}^{\mathfrak q}} \Bigr)$.\\

\textit{Proof of Ineq. 4):}
$L \Bigl( \frac{C_{ii}^{\mathfrak q}}{C_{jj}^{\mathfrak q}} \Bigr)$
$= L \Bigl( C_{ii}^{\mathfrak q} \, \frac{1}{C_{jj}^{\mathfrak q}} \Bigr)$
$\stackrel{(5)}{\leq} L \Bigl(\frac{1}{C_{jj}^{\mathfrak q}} \Bigr) \, s \Bigl( C_{ii}^{\mathfrak q} \Bigr) + L \Bigl( C_{ii}^{\mathfrak q} \Bigr) \, s \Bigl( \frac{1}{C_{jj}^{\mathfrak q}} \Bigr)$

$\stackrel{(7)}{\leq} b^{-2} \, L \Bigl(C_{jj}^{\mathfrak q} \Bigr) \, s \Bigl( C_{ii}^{\mathfrak q} \Bigr) + L \Bigl( C_{ii}^{\mathfrak q} \Bigr) \, s \Bigl( \frac{1}{C_{jj}^{\mathfrak q}} \Bigr)$ where $b \in \Real_+$ chosen such that $C_{jj}^\agiii(I) \cap [-b,b] = \emptyset$ (we assume such a $b$ exists). A valid choice certainly is $b := \inf_{t \in I} |C_{jj}^\agiii(t) |$.\\

\textit{Proof of Ineq. 5):} Completely, analogous to proof of 4).\\

\textit{Proof of Ineq. 6):} 
Consequence of Lem. \ref{lem:Liparithmetic}.(10).

\end{proof}

The theorem provides a recipe to find a Lipschitz bound for the collision criterion function given known Lipschitz numbers of the trajectories' means and spatial covariance mappings. 

However, since most equalities are not tight one should attempt to determine Lipschitz numbers directly wherever possible rather than using the inequalities provided in Lem. \ref{lem:Liparithmetic}.
For instance, if one can determine the best Lipschitz number for $L\bigl((C^{\mathfrak q}_{ij})^2\bigr)$ 
directly (e.g. by utilizing Lem \ref{lem:Liparithmetic}.11) this would normally yield a better Lipschitz constant than obtained by expanding into $2 \, s(C^{\mathfrak q}_{ij}) \, 
L\bigl(C^{\mathfrak q}_{ij}\bigr)$ due to application of Lem. \ref{lem:Liparithmetic}.6. 

Examining the terms in the inequalities we notice the occurrence of suprema of covariances $s(C_{ij})$ or inverted covariances of the form $s(\frac{1}{C_{ii}})$. The latter requires non-vanishing uncertainty in our model. Furthermore, note, the need to evaluate know the extrema is not to burdensome as they can be rapidly found by pre-existing Lipschitz optimizers which are highly efficient. However, in many cases the optima are known a priori. For instance, if one knows that the uncertainty monotonously increases over time we have e.g. 
$s(C_{ij}^\agiii) = C_{ij}^\agiii(\inf I)$ and $s(\frac{1}{C_{ij}^\agiii}) = C_{ij}^\agiii(\sup I)$. Alternatively, the covariances may allow for an analytic closed-form solution of the extremum which may be analytically derived before run-time.

We will revisit these issues in Sec. \ref{sec:coll_detect_ex1} where we examine a concrete application of the theorem to a multi-agent control scenario.

\subsection{A Lipschitz number for the criterion function of our feedback-controlled agents}
\label{sec:coll_detect_ex1}

Let $\agi \in \agset$ be an agent with controlled plant dynamics given by the Ito-SDE \[d\state^\agi = K (\xi^\agi -\state^\agi ) dt + B \, dW. \] 
Here $\state^\agi(t) \in \Real^D$ denotes the agent $\agi's$ state (e.g. location),  $\xi^\agi(t) \in \Real^D$ is the agent's
setpoint signal at time $t \in I=[t_0,t_f]$. Furthermore, $K = \text{diag}(k_1,...,k_D) > 0$ is the controller's gain matrix and $B = \text{diag}(\sqrt \nu_1,...,\sqrt \nu_D)$ reflects the magnitude 
of the uncertainties (disturbances). Let uncertain start state $\state^\agi(t_0)$ be a normally distributed random vector. 
Assume $\agi$'s plan is $p^\agi = \bigl((t_i^\agi,\zeta_i^\agi)\bigr)_{0=1}^{H^\agi}$ where each $\zeta^\agi_{i} = (\zeta^\agi_{i,j})_{j=1}^D$. Let $t \geq t_0$, $\underline i = \arg\max_i \{t_i \leq t_0 \}$,  $\overline i = \arg\min_i \{t_i \geq t \}$ and $\mathcal I := \{ i \in \{1,...,H^\agi\} |  \underline i < i \leq \overline i  \}$. Furthermore, let $\xi^\agi_j$ denote the jth component of step-function reference signal $\xi^\agi$.

For ease of notation we will drop the agent superscripts throughout the remainder of this subsection.
The solution to agent $\agi$'s SDE
is a Gaussian process with vector-valued mean function $\mu^\agi: [t_0,t_f] \to \Real^D$ and matrix-valued covariance function $C^\agi: [t_0,t_f]^2 \to \Real^{D\times D}$. 
By applying Ito-calculus to the suitable expectations of the SDE we can show that we have   
\[\mu_j^\agi(t) = e^{k_j^\agi (t_0-t)} \, \expect{\state_j^\agi(t_0) } + k_j^\agi \, e^{-k_j^\agi t} \int_{t_0}^t e^{k_j^\agi \tilde t} \xi_j^\agi(\tilde t) d \tilde t. \]
The covariance matrix function is $(s,t) \mapsto \text{diag}(\text{cov}_{11}^\agi(s,t),...,\text{cov}_{DD}^\agi(s,t))$ where 
\begin{align*}
\text{cov}_{jj}^\agi(s,t) &= e^{-k_j^\agi (t+s - 2 t_0)} (\expect{\bigr(\state_j^\agi(t_0) \bigl)^2} -  \expect{\state_j^\agi(t_0)}^2 )  \\
&+\frac{\nu_j^\agi}{2 k_j^\agi }  [e^{-k_j^\agi \abs{t-s}} - e^{k_j^\agi ( 2 t_0 - (s+t))}].
\end{align*}

Using the notation of Thm. \ref{thm:Lipschitzno_whittle_2d} we have 

$C_{jj}^\agi (t) = \text{cov}_{jj}^\agi(t,t)$
\begin{align*}
&= e^{-k_j^\agi 2(t - t_0)} (\expect{\bigl(\state_j^\agi(t_0)\bigr)^2} -  \expect{\state_j^\agi(t_0)}^2 ) 
+\frac{\nu_j^\agi}{2 k_j^\agi }  [1- e^{k_j^\agi 2 (  t_0 - t)}] \\
&= e^{-k_j^\agi 2(t - t_0)} \bigl(C_{jj}^\agi(t_0) - \frac{\nu_j^\agi}{2 k_j^\agi } \bigr) + \frac{\nu_j^\agi}{2 k_j^\agi }.  
\end{align*} where $C_{jj}^\agi (t_0)$ is assumed to be a known quantification of the initial state uncertainty.

Next we will derive Lipschitz constants for the the component means and covariances which is necessary to derive a Lipschitz number for the 
collision criterion function. 

Firstly, we consider the mean function. Defining $v^\agi(t) :=  e^{k_j^\agi (t_0-t)} \, \expect{\state_j^\agi(t_0) } $, 
$w^\agi(t) := k_j^\agi \, e^{-k_j^\agi t} $ and $\dot q^\agi(t) :=  e^{k_j^\agi t} \xi_j^\agi( t) $ we
we can restate the component mean function $m_j^\agi$ as 

$\mu_j^\agi (t) = v_j^\agi(t) + w_j^\agi(t) \, q_j^\agi(t)$. 

Leveraging Lem. \ref{lem:Liparithmetic} we see that 
\begin{align} L\bigl(\mu_j^\agi\bigr) 
&\leq L(v_j^\agi) + s(w_j^\agi) \, L(q_j^\agi) + s(q_j^\agi) \, L(w_j^\agi)\\
&\leq s(\dot v_j^\agi) + s(w_j^\agi) \, s(\dot q_j^\agi) + s(q_j^\agi) \, s(w_j^\agi)
\end{align}
 where as before $s(f) = \sup_{t \in I} |f(t)|$ for any function $f$.
Evaluation of the suprema depends on the setpoint signal $\xi$ and on the $k_j$.
For instance, choosing a constant setpoint $\xi$ and $k_j >0$ would yield:

\begin{itemize}
\item $s(\dot v_j^\agi) = \sup_{t \in [t_0,t_f]} |-k_j^\agi \expect{\state_j^\agi(t_0)} e^{k_j^\agi t_0} e^{-k_j^\agi t}|$ 
$=  |k_j \expect{\state_j^\agi(t_0)}|$ where the last equality holds since $e^{-k_j^\agi t}$ decreases monotonically, 

\item $s(w_j^\agi) = \sup_{t \in [t_0,t_f]} |k_j^\agi e^{-k_j^\agi t}| = |k_j^\agi| e^{-k_j^\agi t_0}$,

\item $s(\dot q_j^\agi) =  \sup_{t \in [t_0,t_f]} |\xi_j^\agi \, e^{k_j^\agi t} | = |\xi_j^\agi|\, e^{k_j^\agi t_f},  $ 

\item $s(q_j^\agi) =  \sup_{t \in [t_0,t_f]} |\int_{t_0}^t e^{k_j^\agi \tilde t} \xi_j^\agi d \tilde t|  = |\frac{\xi_j^\agi}{k_j^\agi} [e^{k_j^\agi  t_f} - e^{k_j^\agi t_0}]|$. 
Here we leveraged the monotonicity of the exponential function.

Next, we derive Lipschitz constants for the covariances:

Note the cross-covariances are zero $C_{ij}^\agi (t) = 0, \forall t, i\neq j$.   
Fortunately, the diagonal of the covariance matrix function are also continuously differentiable. In particular, we have     
  $\dot C_{jj}^\agi(t) = -k_j^\agi 2 e^{-k_j^\agi 2(t - t_0)} \bigl(C_{jj}^\agi(t_0) - \frac{\nu_j^\agi}{2 k_j^\agi } \bigr)$. 
We can once again utilize Lem. \ref{lem:Liparithmetic} yielding a Lipschitz bound 
\begin{align*}
  L(C_{jj}^\agi ) &\leq s\bigl(\dot C_{jj}^\agi \bigr)\\
 &\leq  \sup_{t \in [t_0,t_f]} |k_j^\agi 2 \bigl(C_{jj}^\agi(t_0) - \frac{\nu_j^\agi}{2 k_j^\agi } \bigr)  | e^{-k_j^\agi 2(t - t_0)}  \\
 & =  |k_j^\agi 2 \bigl(C_{jj}^\agi(t_0) - \frac{\nu_j^\agi}{2 k_j^\agi } \bigr)  |. 
\end{align*}
 where the last equality follows from the fact that $t \mapsto \exp(-k_j 2(t - t_0))$ is monotonically decreasing.

\end{itemize}

In summary, we have found 

\begin{align}
 L(C_{jj}^\agi ) & \leq |k_j^\agi 2 \bigl(C_{jj}^\agi(t_0) - \frac{\nu_j^\agi}{2 k_j^\agi } \bigr)  |,\\
 L(C_{12}^\agi )  &  = L(C_{21}^\agi)=0, \label{eq:covzero}\\
 L(\mu_j^\agi) &\leq s(\dot v_j^\agi) + s(w_j^\agi) s(\dot q_j^\agi) + s(q_j^\agi) s(w_j^\agi),\\
 =&  |k_j^\agi \expect{\state_j^\agi(t_0)}| + |k_j^\agi| |\xi_j^\agi| \, e^{k_j^\agi (t_f-t_0)} + |\xi_j^\agi \, [e^{k_j^\agi  (t_f-t_0)} - 1]|.
\end{align}

%

Next, we combine our estimates of the mean and covariances with Lem. \ref{lem:Liparithmetic} and Thm. \ref{thm:Lipschitzno_whittle_2d}
to derive a Lipschitz number for the criterion function defined in Thm. \ref{def:collcritfct2d}. 
Let $\mathfrak q \in \{\agi,\agii\}$.
Since $L(C_{12}^\agiii ) = 0$ we have $g_i(t) :=  C_{ii}^{\mathfrak q}(t) +\sqrt{ (C_{ii}^{\mathfrak q})^2  - C_{ii}^{\mathfrak q} \frac{\bigl( C_{ij}^\agiii \bigr)^2}{C_{jj}^{\mathfrak q}} } = 2 C_{ii}^{\mathfrak q}(t)$. By Thm. \ref{thm:Lipschitzno_whittle_2d}.2, this implies 
$L(\alpha_i^\agiii) \leq \sqrt{\frac{1}{2 \delta^\agiii}} \frac{1}{\iota(g_i)} L( g_i)$ 
$= \sqrt{\frac{1}{2 \delta^\agiii}} \frac{1}{\iota(2C_{ii}^\agiii)}  L(2 C_{ii}^\agiii) $
$= \sqrt{\frac{1}{2 \delta^\agiii}} \frac{1}{\iota(C_{ii}^\agiii)}  L( C_{ii}^\agiii) $ where the last equality 
is due to Lem. \ref{lem:Liparithmetic}.3 and due to the fact that $\inf_t |r \, f(t)| = |r| \inf_t |f(t)| $ for all constants $r$, functions $f$.
Next, we determine $\iota(C_{ii}^{\mathfrak q})$. By inspecting its derivative, we notice that $C_{ii}^{\mathfrak q}$ is strictly monotonously increasing iff $C_{ii}^{\mathfrak q}(t_0) - \frac{v_i^{\mathfrak q}}{2 k_i^{\mathfrak q}} < 0$ and monotonously decreasing otherwise. 
Also, $C_{ii}^{\mathfrak q} $ does not attain negative values implying $C_{ii}^{\mathfrak q} = \abs{C_{ii}^{\mathfrak q}}$. 

Hence, $\iota(C_{ii}^{\mathfrak q}) =
\inf\{|C_{ii}^{\mathfrak q} (I) |\} = \inf\{C_{ii}^{\mathfrak q} (I) \}$ 

$= \begin{cases}
C_{ii}^{\mathfrak q} (t_0), \text{ if } C_{ii}^{\mathfrak q}(t_0) < \frac{v_i^{\mathfrak q}}{2 k_i^{\mathfrak q}}; \\
C_{ii}^{\mathfrak q} (t_f), \text{ if } C_{ii}^{\mathfrak q}(t_0) \geq \frac{v_i^{\mathfrak q}}{2 k_i^{\mathfrak q}};
\end{cases}
$
Now, we have all the necessary ingredients to utilize Thm. \ref{thm:Lipschitzno_whittle_2d} in order to wrap-up:\\

$L(\gamma^{\agi,\agii}) \leq \max_{j}  \{L(\mu_j^\agi) + L(\mu_j^\agii) + L(b_j^{\agi,\agii}) \}$\\
$\leq \max_{j} \{L(\mu_j^\agi) + L(\mu_j^\agii) +\frac{1}{2} L(\alpha_j^\agi )+ \frac{1}{2} L(\alpha_j^\agii ) \}$

$\leq \max_j \bigl\{ |k_j^\agi \mu^\agi_j(t_0)| + |k_j^\agi| |\xi_j^\agi| e^{k_j^\agi (t_f-t_0)} + |\xi_j^\agi
\, [e^{k_j^\agi  (t_f-t_0)} - 1]| + |k_j^\agii \mu^\agii_j(t_0)| + |k_j^\agii| |\xi_j^\agii| e^{k_j^\agii (t_f-t_0)} + |\xi_j^\agii
\, [e^{k_j^\agii  (t_f-t_0)} - 1]| +\frac{1}{2} \sqrt{\frac{1}{2 \delta^\agi}} \frac{1}{\iota(C^\agi_{jj})}  L( C^\agi_{jj}) + \frac{1}{2} \sqrt{\frac{1}{2 \delta^\agii}} \frac{1}{\iota(C^\agii_{jj})}  L( C^\agii_{jj}) \bigr\}$.

We can see that this Lipschitz number might adopt large values in certain parts of the domain. Therefore, it might be helpful to recompute the Lipschitz numbers adaptively for different parts of the domain.

\section{Utilising priors encoding belief over change of sign of a criterion function} 
Detection of collisions is based on excluding the possibility of negative criterion function values. However, as these functions are non-convex any numerical 
procedure executed on a digital computer has to achieve this with only a finite number of function evaluations. Given this, what is our confidence 
in not having missed a negative criterion function value?

Thus far, we have proposed using a knowledge (i.e. a prior) about a Lipschitz number of the criterion function to rule out collisions in continuous time based on a finite number of samples. In addition to the Lipschitz-based method presented above, we will now consider an alternative method that assumes a prior belief about the anticipated change of sign of a criterion function.

Before commencing it will prove helpful to introduce the notion of a \textit{sign change point (SCP)}. An SCP is a time step which is the border between two changes in sign of a function. More precisely, time $t$ is an SCP of function $f$ if there exist open intervals $I':=(t',t)$ and $I'':=(t,t'')$ such that $\sgn(f(\tau')) \neq \sgn(f(\tau'')), \forall \tau' \in I', \tau'' \in I''$. 

To give an example, consider the function \[\phi: t \mapsto -t \, \chi_{\Real_-}(t) + 0 \, \chi_{[0,1]}(t)+ (t-1) \chi_{\Real_{>1}}(t). \] As before, $\chi_S$ denotes the indicator function of set $S$. Function $\phi$ has exactly two SCPs -- at $t= 0$ and $t =1$.

Resuming with our discussion, assume we are given $f(t_0),...,f(t_k)$ on a lattice of times $(0 =t_0 <... < t_k = T)$. 
If $f(t_i) \leq 0$ for some $t_i$ we 
will want to conservatively assume a collision has occurred. On the other hand, if all evaluations are positive we desire to specify our confidence that all intermittent unobserved values are. This is the case if no SCPs occur. The presence of an odd number of SCPs between two time steps $t_i, t_{i+1}$ is detectable by checking $\sgn(f(t_i)) \neq \sgn(f(t_{i+1}))$. In fact, if the total number $n$ of SCPs in $[0,T]$ is odd we will detect a change of sign. By contrast, if an even number of SCPs occur we have $\sgn(f(t_i)) = \sgn(f(t_{i+1}))$ and hence, will be oblivious of negative function values in the interval $(t_i,t_{i+1})$.

Now, assume we are given a distribution $Q : \nat \to [0,1]$ representing our belief over the number of occurring SCPs. By the law of total probability, our belief that we will miss the existence of a collision is 
\begin{equation} 
\label{eq:lattice_coll_miss_prob}
\sum_{n \in 2\nat} P_{n,k} Q[n]
\end{equation}
 where $P_{n,k}$ denotes the probability of missing the existence of a collision during collision detection given that $n$ SCPs occur in the interior of the lattice.

In preparation of the next theorem we need the following result:
%
%
%
%
%
%

\begin{lem}[Improved bound] \label{lem:nopartitionsintopairs_improved}
Given a set $S = \{s_1,...,s_n\}$ of $n \in 2\nat$ objects let $\mathcal P$ be the number of ways the set can be partitioned into pairs, i.e. $\mathcal P = |\{ \{\mathfrak P_1,...,\mathfrak P_{n/2}\} | \forall i\neq j: \mathfrak P_i \cap \mathfrak P_j = \emptyset, \bigcup_i \mathfrak P_i = \{s_1,...,s_n\},   \forall i \exists q \neq r: \mathfrak P_i = \{s_q,s_r\} \} |$.

We have $\mathcal P \leq \binom{n(n-1)/2}{n/2} .$

\begin{proof}
We can create $(n^2 -n)/2 = n(n-1)/2$ distinct sets of the form $\{s_q,s_r\} $ of cardinality two.
That is, $|T| = n(n-1)/2$ where $T = \{ \{s_i,s_j \} | i \neq j, i,j =1,...,n \} $.

To generate a partition $\{\mathfrak P_1,...,\mathfrak P_{n/2}\}$ we need to select a subset of $T$ containing $n/2$ two-element sets. Conservatively (not taking into account that not every n/2 -element subset is an actual partition), this could be done in at most $\binom{n(n-1)/2}{n/2} $ ways.
Hence $\mathcal P \leq \binom{n(n-1)/2}{n/2} $.

\end{proof}
\end{lem}

\begin{thm}
\label{thm:collmissprobdue2evenSCP}
Assume we are given $n \in 2\nat_0$ SCPs $s_1,...,s_n$ whose locations are drawn independently from an identical distribution (drawn i.i.d.). Furthermore, we are given a grid of test points $0=t_0 < t_1 <...< t_{k} =T$ where the intermediate times are chosen such that $\forall i \in \{1,...,n\}, j \in \{1,...,k \}: \Pr[ s_i \in (t_{j-1}, t_j ) ] = 1/k$.
The probability of missing the existence of a collision by looking for non-positive elements in the sample $f(t_0),...,f(t_k)$ is: 
\[P_{n,k} \leq \frac{\mathcal P}{\sqrt{k^n}} \leq \frac{\binom{n(n-1)/2}{n/2}}{\sqrt{k^n}}\] where $\mathcal P$ is a function of $n$ (but not of $k$) as defined in Lemma \ref{lem:nopartitionsintopairs_improved}.
In particular, we have $\lim_{k \to \infty} P_{n,k} = 0$. 
\begin{proof}
We define the sample space $\Omega:= \{ (b_1,...,b_n) \in \{1,...,k\}^n \}$ where each $b_i \in \{1,...,k\}$ denotes the index of the time interval (``bin'') $(t_{b_{i-1}}, t_{b_{{i}}}]$ the $i$th SCP $s_i$ falls into ($i=1,...,n$). Due to the assumption that the assignment of SCP to bin is i.i.d., each sample has equal probability and we can compute $P_{n,k}$ as a Laplace probability.
That is, $P_{n,k} = \frac{|G|}{|\Omega|}$ where $G$ is the set of events describing that no bin contains an odd number of SCPs (because if at least one does contain an odd number we detect the presence of a collision). 

Obviously, $|\Omega| = k^n$.

On the other hand, $G = \{(v_1,...,v_k)| \sum_{j=1}^k v_j = n, \forall j \in \{1,...,k\}: v_j \in \{0,...,n\} \cap 2\nat_0 \}$ where $v_j \in \{0,...,n \}$ denotes the number of SCPs falling into bin $j \in \{1,...,k\}$. We will find an upper bound on $G's$ cardinality by constructing a finite set $H$ for whose cardinality one can easily establish $\sqrt{k}^n \,\mathcal P$ as an upper bound. We show that one can define an injective function $\psi: G \to H$. The latter establishes that $|G| \leq |H|$. Thus, $\frac{|G|}{|\Omega|} \leq \frac{|H|}{|\Omega|} \leq \frac{\mathcal P}{\sqrt{k}^n}$ which will hence conclude the proof. 

We generate $H$ by invoking a two-stage process (where in each stage it is easy to enumerate all possible elements that are generated). 
In the first stage, we partition the SCPs into $n/2$ pairs (which we always can since we assumed $n$ to be even). In the second stage, we assign these pairs to the bins in which the pairs are merged into joint sets of SCPs. Therefore, $H = \{(M_1,...,M_k) \vert \, M_1 \subset \{s_1,...,s_n\},...,M_k  \subset \{s_1,...,s_n\}, |M_1|,...,|M_k| \in 2\nat_0, |M_1|+...+|M_k| =n, H = \bigcup_i M_i  \}$.

Let $\mathcal P$ be the number of ways in which one can partition the $n$ SCPs into $n/2$ (unordered) pairs 
 i.e. $\mathcal P = |\{ \{\mathfrak P_1,...,\mathfrak P_{n/2}\} | \forall i\neq j: \mathfrak P_i \cap \mathfrak P_j = \emptyset, \bigcup_i \mathfrak P_i = \{s_1,...,s_n\},   \forall i \exists q \neq r: \mathfrak P_i = \{s_q,s_r\} \} |$. (Cf.  Lem. \ref{lem:nopartitionsintopairs_improved} for a bound). 
In the second stage, the pairs are distributed among the $k$ bins (intervals) (which can be done in $\mathcal B = k^{n/2}$ ways) before the sets within each bin are merged. The number of different paths the process can take to generate an element in $H$ is $\mathcal P$ (number of partitions into pairs) multiplied with $\mathcal B$ (number of ways the pairs constituting the partition can then be distributed into the bins).

By construction, each final assignment (subsets of SCPs to bins) generated by the two-stage process is an element of $H$. Conversely, let $(M_1,...,M_k) \in H$ then it is easy to verify it could be generated by the two-stage process (however, there may be multiple paths in the process generating the same element of $H$). Hence, $|H| \leq \mathcal P \, \mathcal B$. 

We finalize our considerations by defining the function $\psi: G \to H, (v_1,...,v_n) \mapsto (M_1,...,M_k)$ where $M_1 := \{s_1,...,s_{v_1}\}$ and for $i>1$: $ M_i := \{s_{1+w_i},...,s_{v_i+w_i} \}$ where $w_i=\sum_{j < i} v_j$. It is easy to see that $\psi$ maps different $(v_1,...,v_n) \in G$ to different $(M_1,...,M_k) \in H$. Hence, $\psi$ is injective. Since both $H$ and $G$ are finite this implies $|G| \leq |H|$. 
Wrapping up,  $\frac{|G|}{|\Omega|} \leq \frac{|H|}{|\Omega|} \leq \frac{\mathcal P \mathcal B}{k^n} \leq \frac{\mathcal P k^{n/2}}{k^n}$ $ =  \frac{\mathcal P }{k^{n/2}} \stackrel{Lem. \ref{lem:nopartitionsintopairs_improved}}{\leq} \frac{\binom{n(n-1)/2}{n/2}}{\sqrt{k^n}}$.
\end{proof}

\end{thm}

In conjunction with Eq. \ref{eq:lattice_coll_miss_prob}, Thm. \ref{thm:collmissprobdue2evenSCP} can provides a recipe how to do collision detection based on a finite number of criterion function samples such that our confidence in not having overlooked an existing collision is above a certain threshold $\theta$.
To this end we require a prior $Q$ over the number of SCPs. Let $N$ be a random variable quantifying the number of occurring SCPs and let $M$ denote the event that we would not detect the existence of an SCP.

By the law of total probability we have 
\begin{align}
\Pr[M | k] &=   \sum_{n\in \nat_0} \Pr[M | N=n, k ] \,Q[N=n]  \\&= \sum_{n\in 2\nat_0} \Pr[M | N=n, k ]\, Q[N=n] 
\\&= \sum_{n\in 2\nat_0} P_{n,k} \,Q[N=n] 
\label{eq:probmissSCP}
\end{align}
where the last equality follows from $ \Pr[M | N=n, k ] =0 , (n \in 2\nat-1), k\geq 1$.

Given a threshold $\theta \in (0,1)$ we can then utilize Thm. \ref{thm:collmissprobdue2evenSCP} to choose a lattice resolution $k$ such that 
\[\Pr[M | k] \stackrel{!}{<} \theta. \] 

To illustrate this we provide to simple examples:

\begin{ex}[Finite-support prior]
As a simplistic example, assume we desire to detect a collision between two agents $\agi,\agii$ with plans containing two setpoints each (stabilizing their start and goal state each) and with linear dynamics. We know that at their start and goal locations, no collisions occur. Given this our belief over the number $n$ of SCPs may be $Q[N=0] = 0.5$ $Q[N=1] =.1, Q[N=2]=.4$. Furthermore, in the given time interval [0,T] we assume a flat prior distribution giving rise to an equidistant lattice $0=t_0 < t_1 <...< t_{k} =T$ of samples during collision detection. That is, $t_{i+1} - t_i = T/k$ for some $k$. 
Our collision detection method now inspects our criterion function values $\Gamma^{\agi,\agii}(t_0),...,\Gamma^{\agi,\agii}(t_k)$. If all values are positive we assume no collision has occurred. How large should we set $k$ in order to ensure that our confidence in this assertion is at least $\theta \in (0,1)$?
By Thm. \ref{thm:collmissprobdue2evenSCP}, we know that the probability of having missed a collision with this simple method is less than $ \frac{0.4}{\sqrt{k^2}}$. Therefore, we need to set $k$ such that $1- \frac{0.4}{\sqrt{k^2}} \geq \theta$ which is equivalent to setting the number $k$ of criterion function evaluations to $k \geq \frac{0.4}{1-\theta}.  $ 
\end{ex}

\end{document}